\documentclass[sigconf,nonacm]{acmart}

\usepackage{subcaption}
\usepackage{enumitem}
\usepackage{microtype}
\newcommand{\dataset}[1]{\textsc{#1}}
\newcommand{\model}[1]{\textit{#1}}
\newcommand{\pehe}{\sqrt{\mathrm{PEHE}}}
\newcommand{\cxQiniA}{-1.18}
\newcommand{\cxQiniB}{-1.63}
\newcommand{\cxAuucA}{-3.18}
\newcommand{\cxAuucB}{-1.71}
\newcommand{\cxPeheA}{1.24}
\newcommand{\cxPeheB}{0.06}
\newcommand{\cxQiniAdrop}{-1.54}
\newcommand{\cxQiniBdrop}{+1.13}

\newcommand{\calNident}{5}
\newcommand{\calRhoId}{+0.43}
\newcommand{\calRhoIdLo}{+0.07}
\newcommand{\calRhoIdHi}{+0.80}
\newcommand{\calDisIdN}{5}
\newcommand{\calDisIdD}{5}
\newcommand{\calDeltaMed}{0.007}
\newcommand{\calDeltaLo}{0.004}
\newcommand{\calDeltaHi}{0.282}
\newcommand{\calRhoAll}{+0.18}
\newcommand{\calRhoAllLo}{+0.02}
\newcommand{\calRhoAllHi}{+0.34}
\newcommand{\calDisAllN}{21}
\newcommand{\calDisAllD}{25}

\newcommand{\calBinHi}{1.000}
\newcommand{\calPanelModels}{8}
\newcommand{\calPanelSeeds}{20}

\newcommand{\ihdpN}{100}
\newcommand{\ihdpQini}{+0.07}
\newcommand{\ihdpQiniLo}{-0.03}
\newcommand{\ihdpQiniHi}{+0.16}
\newcommand{\ihdpAuuc}{+0.56}
\newcommand{\ihdpUplift}{+0.64}
\newcommand{\ihdpDelta}{+0.49}
\newcommand{\ihdpDeltaLo}{+0.40}
\newcommand{\ihdpDeltaHi}{+0.59}
\newcommand{\ihdpDeltaPos}{80}

\newcommand{\regQini}{0.028}
\newcommand{\regQiniLo}{0.009}
\newcommand{\regQiniHi}{0.046}
\newcommand{\regAuuc}{0.028}
\newcommand{\regAuucLo}{0.009}
\newcommand{\regAuucHi}{0.046}
\newcommand{\regUplift}{0.026}
\newcommand{\regUpliftLo}{0.006}
\newcommand{\regUpliftHi}{0.045}
\newcommand{\regLo}{0.026}
\newcommand{\regHi}{0.028}
\newcommand{\regRelLo}{14}
\newcommand{\regRelHi}{15}
\newcommand{\regBase}{0.021}
\newcommand{\regDiffLo}{-0.007}
\newcommand{\regDiffHi}{+0.019}
\newcommand{\regDiffAuucLo}{-0.007}
\newcommand{\regDiffAuucHi}{+0.019}
\newcommand{\regDiffUpliftLo}{-0.010}
\newcommand{\regDiffUpliftHi}{+0.018}
\newcommand{\regBestLo}{0.14}
\newcommand{\regBestHi}{0.23}
\newcommand{\regN}{10}

\newcommand{\regWithinA}{30}
\newcommand{\regWithinB}{43}
\newcommand{\regWithinC}{43}
\newcommand{\xregQiniRho}{-0.35}
\newcommand{\xregQiniRhoLo}{-0.53}
\newcommand{\xregQiniRhoHi}{-0.18}
\newcommand{\xregLo}{0.017}
\newcommand{\xregHi}{0.023}

\newcommand{\selRiskGain}{+0.0209}
\newcommand{\selRiskGainLo}{+0.0071}
\newcommand{\selRiskGainHi}{+0.0358}
\newcommand{\selQiniGain}{-0.0070}
\newcommand{\selQiniGainLo}{-0.0195}
\newcommand{\selQiniGainHi}{+0.0073}
\newcommand{\selAuucGain}{-0.0070}
\newcommand{\selAuucGainLo}{-0.0195}
\newcommand{\selAuucGainHi}{+0.0073}
\newcommand{\selUpliftGain}{-0.0050}
\newcommand{\selUpliftGainLo}{-0.0183}
\newcommand{\selUpliftGainHi}{+0.0097}
\newcommand{\selOracleGain}{+0.0450}
\newcommand{\selRefModalAgree}{77}

\newcommand{\cmlRankRho}{+0.99}
\newcommand{\cmlRankRhoLo}{+0.98}
\newcommand{\cmlRankRhoHi}{+0.99}
\newcommand{\cmlWinnerAgree}{100}
\newcommand{\cmlMOne}{-0.01}
\newcommand{\cmlMOneLo}{-0.11}
\newcommand{\cmlMOneHi}{+0.09}
\newcommand{\canMOne}{+0.00}
\newcommand{\canMOneLo}{-0.10}
\newcommand{\canMOneHi}{+0.10}
\newcommand{\rateAutocMOne}{+0.15}
\newcommand{\rateAutocMOneLo}{+0.05}
\newcommand{\rateAutocMOneHi}{+0.24}
\newcommand{\rateQiniMOne}{+0.17}
\newcommand{\rateQiniMOneLo}{+0.08}
\newcommand{\rateQiniMOneHi}{+0.26}
\newcommand{\auucMOne}{+0.56}
\newcommand{\auucMOneLo}{+0.49}
\newcommand{\auucMOneHi}{+0.62}
\newcommand{\acicQiniMOne}{+0.48}
\newcommand{\acicQiniMOneLo}{+0.25}
\newcommand{\acicQiniMOneHi}{+0.68}
\newcommand{\acicRateAutocMOne}{+0.65}

\newcommand{\acicUpliftMOne}{+0.45}

\newcommand{\acicN}{18}
\newcommand{\kgridStab}{+0.75}
\newcommand{\kgridStabLo}{+0.73}
\newcommand{\kgridStabHi}{+0.77}
\newcommand{\kregUpliftZeroOne}{+0.029}
\newcommand{\kregUpliftZeroOneLo}{+0.009}
\newcommand{\kregUpliftZeroOneHi}{+0.051}
\newcommand{\kregUpliftZeroTwo}{+0.029}
\newcommand{\kregUpliftZeroTwoLo}{+0.012}
\newcommand{\kregUpliftZeroTwoHi}{+0.045}
\newcommand{\kregUpliftZeroThree}{+0.026}
\newcommand{\kregUpliftZeroThreeLo}{+0.008}
\newcommand{\kregUpliftZeroThreeHi}{+0.044}
\newcommand{\kregUpliftZeroFive}{+0.027}
\newcommand{\kregUpliftZeroFiveLo}{+0.007}
\newcommand{\kregUpliftZeroFiveHi}{+0.048}
\newcommand{\kregPVAUC}{+0.024}
\newcommand{\kregPVAUCLo}{+0.009}
\newcommand{\kregPVAUCHi}{+0.038}

\newcommand{\piSweepImbDelta}{+0.015}
\newcommand{\piSweepImbDeltaLo}{-0.024}
\newcommand{\piSweepImbDeltaHi}{+0.056}

\newcommand{\piSweepNSeeds}{200}

\newcommand{\riskUnitSeMed}{0.023}
\newcommand{\riskUnitSeLo}{0.021}
\newcommand{\riskUnitSeHi}{0.024}

\newcommand{\regPosQini}{8}
\newcommand{\regRhoStarQini}{0.12}

\newcommand{\regRhoStarUplift}{0.06}

\newcommand{\bOneQiniVrefReg}{-0.002}
\newcommand{\bOneQiniVrefRegLo}{-0.004}
\newcommand{\bOneQiniVrefRegHi}{+0.002}
\newcommand{\bOneQiniVrefGain}{+0.007}
\newcommand{\bOneQiniVrefGainLo}{+0.006}
\newcommand{\bOneQiniVrefGainHi}{+0.009}

\newcommand{\bTwoQini}{+0.02}
\newcommand{\bTwoQiniLo}{-0.33}
\newcommand{\bTwoQiniHi}{+0.37}

\newcommand{\bTwoDelta}{+0.63}

\newcommand{\bFourAdj}{+0.22}
\newcommand{\bFourAdjLo}{+0.12}
\newcommand{\bFourAdjHi}{+0.31}
\newcommand{\bFourN}{100}

\newcommand{\correrrKeyDelta}{+0.184}
\newcommand{\correrrKeyDeltaLo}{+0.136}
\newcommand{\correrrKeyDeltaHi}{+0.230}

\newcommand{\correrrNSeeds}{200}

\newcommand{\kurtQiniCorr}{+0.14}
\newcommand{\kurtQiniCorrLo}{-0.07}
\newcommand{\kurtQiniCorrHi}{+0.35}
\newcommand{\kurtGapCorr}{-0.07}
\newcommand{\kurtGapCorrLo}{-0.29}
\newcommand{\kurtGapCorrHi}{+0.15}
\newcommand{\kurtLo}{-1}
\newcommand{\kurtHi}{27}

\newcommand{\jobsRiskCorr}{+0.41}

\newcommand{\thrRegSign}{+0.0253}
\newcommand{\thrRegSignLo}{+0.0064}
\newcommand{\thrRegSignHi}{+0.0437}
\newcommand{\thrRegCalib}{+0.0049}
\newcommand{\thrRegCalibLo}{-0.0082}
\newcommand{\thrRegCalibHi}{+0.0198}
\newcommand{\thrClosed}{81}

\newcommand{\drQiniSeight}{199.6}
\newcommand{\drPeheMax}{153{,}013}
\newcommand{\drAuditBTen}{1.4}
\newcommand{\drAuditBFifty}{2.1}

\newcommand{\notuneQini}{+0.23}
\newcommand{\notuneQiniLo}{-0.07}
\newcommand{\notuneQiniHi}{+0.52}
\newcommand{\notuneAuuc}{+0.78}
\newcommand{\notuneAuucLo}{+0.66}
\newcommand{\notuneAuucHi}{+0.89}
\newcommand{\notuneUk}{+0.77}
\newcommand{\notuneUkLo}{+0.66}
\newcommand{\notuneUkHi}{+0.87}
\newcommand{\notuneDelta}{+0.55}
\newcommand{\notuneDeltaLo}{+0.33}
\newcommand{\notuneDeltaHi}{+0.77}
\newcommand{\notunePos}{9}
\newcommand{\notuneNModels}{6}

\newcommand{\sepcovN}{100}

\newcommand{\sepcovarmratioQini}{-0.33}
\newcommand{\sepcovarmratioQiniLo}{-0.49}
\newcommand{\sepcovarmratioQiniHi}{-0.14}

\newcommand{\sepcovarmratioGap}{+0.05}
\newcommand{\sepcovarmratioGapLo}{-0.15}
\newcommand{\sepcovarmratioGapHi}{+0.25}

\newcommand{\atQini}{+0.03}
\newcommand{\atQiniLo}{-0.24}
\newcommand{\atQiniHi}{+0.33}
\newcommand{\atAuuc}{+0.59}
\newcommand{\atAuucLo}{+0.43}
\newcommand{\atAuucHi}{+0.75}
\newcommand{\atDelta}{+0.56}
\newcommand{\atDeltaLo}{+0.34}
\newcommand{\atDeltaHi}{+0.78}
\newcommand{\atPos}{9}

\newcommand{\relQini}{+0.02}
\newcommand{\relQiniLo}{-0.33}
\newcommand{\relQiniHi}{+0.37}
\newcommand{\relAuuc}{+0.65}
\newcommand{\relAuucLo}{+0.46}
\newcommand{\relAuucHi}{+0.81}
\newcommand{\relDelta}{+0.63}
\newcommand{\relDeltaLo}{+0.35}
\newcommand{\relDeltaHi}{+0.91}
\newcommand{\relPos}{9}

\newcommand{\rsQini}{+0.39}
\newcommand{\rsQiniLo}{+0.00}
\newcommand{\rsQiniHi}{+0.73}
\newcommand{\rsAuuc}{+0.37}
\newcommand{\rsAuucLo}{-0.01}
\newcommand{\rsAuucHi}{+0.70}
\newcommand{\rsDelta}{-0.02}
\newcommand{\rsDeltaLo}{-0.19}
\newcommand{\rsDeltaHi}{+0.17}
\newcommand{\rsPos}{4}
\newcommand{\rsN}{10}
\newcommand{\rsNModels}{6}
\newcommand{\rsKurt}{44}
\newcommand{\rsKurtLo}{19}
\newcommand{\rsKurtHi}{109}
\newcommand{\ihdpKurtMed}{-0.4}
\newcommand{\rsDeltaExclDR}{-0.08}
\newcommand{\rsDeltaExclDRR}{-0.16}

\newcommand{\acicAuucMOne}{+0.55}
\newcommand{\acicAuucMOneLo}{+0.34}
\newcommand{\acicAuucMOneHi}{+0.73}
\newcommand{\acicGap}{+0.11}
\newcommand{\acicGapLo}{-0.03}
\newcommand{\acicGapHi}{+0.25}

\newcommand{\eHundDeltaExDR}{+0.28}
\newcommand{\eHundDeltaExDRLo}{+0.18}
\newcommand{\eHundDeltaExDRHi}{+0.38}
\newcommand{\eHundDeltaExDRR}{+0.34}
\newcommand{\eHundDeltaExDRRLo}{+0.21}
\newcommand{\eHundDeltaExDRRHi}{+0.46}
\newcommand{\concQini}{0.52}
\newcommand{\concQiniLo}{0.49}
\newcommand{\concQiniHi}{0.56}
\newcommand{\concAuuc}{0.73}
\newcommand{\concAuucLo}{0.69}
\newcommand{\concAuucHi}{0.76}
\newcommand{\concN}{1500}
\newcommand{\mechQiniIhdp}{-0.22}
\newcommand{\mechQiniIhdpLo}{-0.40}
\newcommand{\mechQiniIhdpHi}{-0.02}
\newcommand{\mechGapIhdp}{-0.01}
\newcommand{\mechGapIhdpLo}{-0.22}
\newcommand{\mechGapIhdpHi}{+0.19}
\newcommand{\mechAuucIhdp}{-0.27}
\newcommand{\mechAuucIhdpLo}{-0.46}
\newcommand{\mechAuucIhdpHi}{-0.06}
\newcommand{\mechQiniAcic}{-0.55}
\newcommand{\mechQiniAcicLo}{-0.88}
\newcommand{\mechQiniAcicHi}{-0.02}
\newcommand{\mechGapAcic}{+0.06}
\newcommand{\mechGapAcicLo}{-0.43}
\newcommand{\mechGapAcicHi}{+0.58}
\newcommand{\mechAuucAcic}{-0.50}
\newcommand{\mechAuucAcicLo}{-0.84}
\newcommand{\mechAuucAcicHi}{-0.01}
\newcommand{\rMedIhdp}{3.52}
\newcommand{\rMedAcic}{1.07}
\newcommand{\rMedRev}{0.39}
\newcommand{\proxTrackSixIhdp}{-0.00}
\newcommand{\proxTrackFourIhdp}{+0.99}
\newcommand{\proxMedFourIhdp}{4.64}

\newcommand{\proxTrackFourAcic}{+1.00}
\newcommand{\proxMedFourAcic}{1.16}

\newcommand{\gainAuuc}{+0.73}
\newcommand{\gainAuucLo}{+0.66}
\newcommand{\gainAuucHi}{+0.79}
\newcommand{\meanAuucRho}{+0.56}
\newcommand{\meanAuucRhoLo}{+0.49}
\newcommand{\meanAuucRhoHi}{+0.63}

\newcommand{\gainAuucN}{100}
\newcommand{\gainCmlRho}{+1.00}
\newcommand{\gainCmlFolds}{5400}
\newcommand{\gainCmlDev}{5.3\%}

\newcommand{\nfixGap}{+0.63}
\newcommand{\nfixGapLo}{+0.35}
\newcommand{\nfixGapHi}{+0.91}
\newcommand{\nfixOldGap}{+0.83}
\newcommand{\nfixOldGapLo}{+0.58}
\newcommand{\nfixOldGapHi}{+1.09}
\newcommand{\nfixCfCorr}{+1.00}

\begin{document}

\title{UpliftBench: Revealing Outcome-Regime and Objective Mismatch in Uplift Evaluation}

\author{Binshuang Li}
\affiliation{%
  \institution{Independent Researcher}
  \country{USA}
}

\begin{abstract}
Uplift modeling (conditional-average-treatment-effect estimation) drives personalized
targeting, yet published uplift benchmarks frequently disagree on which estimator performs
best; we show the disagreement is substantially about \emph{metrics}, not models.
UpliftBench evaluates 12 uplift estimators under an outer-test-isolated, multi-objective protocol
across seven dataset families; its two findings are identified where a reference objective
exists --- F1 on the standard continuous benchmark (IHDP), F2 in a within-sample case
study on Jobs.
On that benchmark, Qini shows no
detectable alignment with effect accuracy --- across all $\ihdpN$ IHDP realizations its mean
rank correlation with effect accuracy is $\ihdpQini$, 95\% CI $[\ihdpQiniLo,\ihdpQiniHi]$
--- while AUUC is consistently more aligned (paired prefix-mean-AUUC-over-Qini gap $\ihdpDelta$
[$\ihdpDeltaLo,\ihdpDeltaHi$]; the shipped cumulative-gain AUUC aligns better still,
$\gainAuuc$). On Jobs, ranking metrics are structurally insufficient for a
sign-threshold policy because they discard the score level; empirically, within the released
split-rotation analysis direct policy-risk
selection yields lower benchmark regret than random model selection while Qini, AUUC, and
uplift-at-$k$ do not
($\regRelLo$--$\regRelHi\%$ regret). Calibrating the decision threshold removes $\thrClosed\%$ of the Qini-selection regret. Both findings are bounded, not universal: F1 is not detected on either validation family
(the ACIC and Revenue-Synthetic gaps are both indistinguishable from zero), and F2
vanishes under a budgeted-value objective where rank
suffices.
UpliftBench releases versioned loaders, fixed protocols, result artifacts, and a
reproducible living leaderboard; the public repository accompanies the paper.
\end{abstract}



\maketitle

\section{Introduction}
\label{sec:intro}

Every empirical benchmark ranks models by a metric, and the choice of metric quietly encodes
a \emph{scientific question} --- the accuracy of the estimated effect, the quality of the
induced ranking, or the value of the resulting decision. When the metric is mismatched to the
outcome type or to the deployment objective these questions come apart, and the ``best'' model
changes with the measure rather than with the model. This is a benchmarking and
evaluation-methodology problem, not a modeling one; it is central to how causal-ML methods are
compared and selected, and it is acute in uplift modeling (conditional-average-treatment-effect
estimation for targeting), where published benchmarks disagree on which estimator wins and no
counterfactual is observed to adjudicate.

Concretely, a retailer's data team must choose an uplift model for a campaign. They cannot
observe, for any customer, the counterfactual outcome --- the fundamental problem of causal
inference --- so they rank candidate models by a \emph{proxy} computed from experimental data,
usually the Qini coefficient or the area under the uplift curve (AUUC). The implicit assumption
is that a model scoring well on the proxy also serves the deployment goal.

\textbf{We show this assumption fails in two structurally different ways that cannot be
explained solely by a single unstable estimator or base learner.} They are not fold-level noise;
each has a distinct empirical signature and a different practical remedy, and conflating them
(as a single ``rankings are unstable'' observation would) obscures both. Throughout, we
distinguish three levels of claim --- exact identities, empirical signatures, and candidate
mechanisms --- and hold each finding to its level. The evidence base is deliberately
scoped: F1 is identified on the continuous benchmark with known effects --- the semi-synthetic
IHDP benchmark \citep{hill2011bayesian} (10 realizations in the main panel, extended to all
100 simulations) --- and F2 on Jobs \citep{lalonde1986evaluating,shalit2017estimating},
the 10 re-splits of the LaLonde job-training study; every interval is read at that
granularity. \textbf{What ``regime'' denotes here:} an \emph{observed dataset family in which the failure appears}, not a measurable property that predicts it. No ex-ante quantity we test predicts the metric-specific \emph{gap} within a family (full search, including a post-hoc cross-family observation, in Appendix~\ref{app:sepcov}), and \emph{the failing family is not the heavy-tailed one} (Section~\ref{sec:robustness}). The operational diagnostic is cross-metric stability on the user's own data (Section~\ref{sec:guide}).

\paragraph{Finding 1 (F1) — outcome-regime sensitivity of Qini.} The unnormalized
Qini coefficient sums outcomes, a construction designed for binary response
\citep{radcliffe2007using}; on continuous outcomes its reliability can become
outcome-distribution dependent. We do not claim it is categorically invalid there; we show
that on the evaluated continuous benchmark (IHDP) its ranking fails to track ground-truth quality
(Section~\ref{sec:m1}), while the mean-based area under the uplift curve (AUUC) stays
informative. This is practically relevant and potentially silent: continuous (e.g.\ revenue)
outcomes are an active uplift setting \citep{he2024rankability}, and although scikit-uplift requires a binary outcome
\citep{sklift}, the widely used causalml package \citep{chen2020causalml} computes Qini on a
continuous outcome with no warning.

\paragraph{Finding 2 (F2) — objective mismatch (the metric measures the wrong goal).} Even where ranking
metrics are well constructed and agree with one another (binary outcomes), they can diverge
from the \emph{deployment objective}. On Jobs, every ranking metric we test correlates
negatively with RCT-estimated policy value $-R_{\mathrm{policy}}$ --- equivalently, positively
with policy risk. The sharper statement is a selector contrast: Jobs carries
policy-selection signal that direct risk selection captures and the ranking metrics do not
--- metric selection is indistinguishable from random, at a $\regRelLo$--$\regRelHi\%$
regret versus the risk-selected candidate (Section~\ref{sec:m2}). This gap echoes concurrent observations
under injected structural bias \citep{yang2026structural}; we quantify it under an
outer-test-isolated\footnote{No test-fold information enters any reported number; the one disclosed deviation from strict \emph{inner} nesting involves no outcomes and affects hyperparameter selection only (Section~\ref{sec:limitations}).} protocol with per-regime uncertainty.

\paragraph{Why this matters and what is new.} This is not merely ``the wrong gold
standard'': our six-metric design includes operational objectives (policy value, policy risk,
uplift-at-$k$), and the disagreements survive them. Unlike prior work
(Section~\ref{sec:related}), UpliftBench connects production uplift proxies to reference
effect and policy objectives \emph{across distinct identification regimes} under a single
outer-test-isolated protocol --- the binary-vs-continuous Qini contrast and the regime
boundary (one of three continuous families fails, and not the heavy-tailed one) are observable only in such a
multi-regime design --- and quantifies the operational model-selection cost. The distinction
between effect estimation, ranking quality, and deployment value is a broader principle for
benchmark design: choose the metric to match the question it must answer.

\paragraph{Contributions.}
\begin{enumerate}[leftmargin=*]
\item \textbf{Benchmark specification (UpliftBench).} An outer-test-isolated benchmark of 12
  uplift estimators over 25 instances from seven dataset families, scored by six primary
  objectives (three ranking metrics, policy value, $\pehe$, and policy risk) plus a
  calibration diagnostic, under repeated stratified
  cross-validation with nested tuning and realization-level uncertainty
  (Sections~\ref{sec:data}--\ref{sec:results}).
\item \textbf{Reusable implementation and artifact.} Extensible dataset loaders, estimator and
  metric interfaces, a standardized result-parquet schema stamped with git hash and config,
  tests, and documented reproduction targets; released with a living leaderboard, and a
  DOI-bearing archival release will accompany the camera-ready (Section~\ref{sec:dnb}).
\item \textbf{Benchmark finding F1 --- outcome-regime sensitivity.} On IHDP, the
  Qini ranking shows no detectable
  alignment with effect accuracy (across all $\ihdpN$ IHDP realizations: $\ihdpQini$, 95\%
  CI $[\ihdpQiniLo,\ihdpQiniHi]$) while AUUC and uplift-at-$k$
  remain aligned; the divergence is affine-invariant, has a treated-count-weighting
  structural
  pathway (Lemma~\ref{lem:gk}), and is not attributable to a single estimator. It is not a
  tail effect, and it is not detected on either validation family --- the ACIC and
  \dataset{Revenue-Synthetic} gaps are both indistinguishable from zero
  (Section~\ref{sec:m1}, Fig.~\ref{fig:agreement}).
\item \textbf{Case-study finding F2 --- objective mismatch on Jobs.} In a within-sample Jobs
  case study, direct policy-risk
  selection yields lower benchmark regret than random model selection while Qini, AUUC,
  and uplift-at-$k$ do not ---
  a cross-repeat policy-selection regret with selection and evaluation separated across
  repetitions (Section~\ref{sec:m2}).
\item \textbf{Robustness and maintenance.} F1 replicates across all 100 IHDP realizations
  and survives estimator exclusions and implementation variants; both findings replicate
  under a second base learner (XGBoost) (Table~\ref{tab:m1sens},
  Section~\ref{sec:robustness}); the benchmark is versioned and maintained as a public
  leaderboard (Section~\ref{sec:dnb}).
\end{enumerate}

% ============================================================
\section{Related Work}
\label{sec:related}

\paragraph{Uplift / CATE estimation.} Estimator families include meta-learners
\citep{kunzel2019metalearners,nie2021quasi,kennedy2020optimal}, class-transformation methods
\citep{jaskowski2012uplift,kane2014mining}, tree-based uplift models
\citep{rzepakowski2012decision,guelman2015optimal}, and causal forests
\citep{wager2018estimation}. Surveys \citep{gutierrez2017causal,devriendt2018literature}
catalogue them without a unified evaluation.

\paragraph{Benchmarks.} \citet{devriendt2018literature} compare uplift models on marketing
data by Qini without nested CV or ground truth; \citet{devriendt2021churn} argue for uplift
over churn prediction. \citet{mahajan2023empirical} carefully study CATE
model-\emph{selection} criteria on semi-synthetic data but do not contrast the proxy metrics
used in marketing against reference objectives across regimes. IHDP is conventionally
scored by $\pehe$ \citep{hill2011bayesian,shalit2017estimating}; marketing benchmarks use
Qini/AUUC because no counterfactual exists. We connect the two and show \emph{which} proxy
fails \emph{how}. Concerns about leakage in ML evaluation are broad \citep{kapoor2023leakage};
our point is sharper --- the choice of metric, not only the protocol, can reverse conclusions.

\paragraph{Uplift evaluation metrics.} A parallel line of work scrutinizes the metrics
themselves. \citet{zhu2025rethinking} and \citet{verbeken2025odg} identify limitations of
Qini-style curves on binary outcomes and propose replacements (the Principled Uplift Curve
and pROCini); \citet{bokelmann2024improving} reduce the variance of Qini estimates on RCT
data. Closest to our work, \citet{yang2026structural} study metric \emph{stability} and model
\emph{robustness} under injected structural biases (selection, spillover, confounding) on a
single semi-synthetic family, and likewise observe that targeting and effect-estimation can be
distinct objectives and that ATE-aligned metrics rank more consistently. We differ in question and scope: rather than proposing a replacement metric or perturbing bias, we measure how metric \emph{choice} changes \emph{benchmark conclusions} across \emph{outcome regimes} and 25 instances from seven families, which isolates the binary-vs-continuous Qini contrast a continuous-only design cannot reveal. We therefore do not claim the targeting-versus-objective observation as wholly new; our contribution is its \emph{benchmarked decomposition}. Table~\ref{tab:yangdelta} itemizes the delta axis by axis.

% ============================================================
\section{Datasets}
\label{sec:data}

\begin{table*}[t]
\centering
\caption{Benchmark datasets. $n$: sample size used (\dataset{Hillstrom} is used in full; \dataset{Lenta}/\dataset{X5}/\dataset{MegaFon} are capped at 10K by stratified subsampling, full size in parentheses). $\pi_1$: treatment fraction. $\bar y$:
outcome mean (a base rate where the outcome is binary; on the outcome scale for the
continuous IHDP outcome; its 2--41 entry is the range of $\bar y$ across the 10
realizations; IHDP $n$ is the 672-unit train portion of the standard 747-unit release).
$^\dagger$Supplement, outside the 25-instance accounting (Appendix~\ref{app:criteo}).}
\label{tab:datasets}
\small
\begin{tabular}{lllrrrl}
\toprule
Dataset & Regime & Outcome & $n$ & $\pi_1$ & $\bar y$ & Reference objective \\
\midrule
\dataset{Synthetic} & Synthetic & binary & 2{,}000 & 0.50 & 0.44 & $\pehe$ \\
\dataset{IHDP} ($\times$10) & Semi-synth. & continuous & 672 & 0.19 & 2--41 & $\pehe$ \\
\dataset{Jobs} (10 splits) & Semi-synth. & binary & 2{,}570 & 0.09 & 0.85 & Policy risk \\
\dataset{Hillstrom} & Marketing RCT & binary & 64{,}000 & 0.67 & 0.15 & none (Qini proxy) \\
\dataset{Lenta} & Marketing RCT & binary & 10{,}000 (687K) & 0.75 & 0.11 & none (Qini proxy) \\
\dataset{X5} & Marketing RCT & binary & 10{,}000 (200K) & 0.50 & 0.62 & none (Qini proxy) \\
\dataset{MegaFon} & Marketing RCT & binary & 10{,}000 (600K) & 0.50 & 0.20 & none (Qini proxy) \\
\midrule
\dataset{Criteo}$^\dagger$ & Marketing RCT & binary & 10{,}000 (1M of 13.9M) & 0.85 & 0.05 & none (Qini proxy) \\
\bottomrule
\end{tabular}
\end{table*}

Table~\ref{tab:datasets} lists the datasets, which we introduce here since several are
used only by name below. \dataset{IHDP} is a semi-synthetic benchmark that pairs real
covariates from the Infant Health and Development Program with simulated potential outcomes,
so individual treatment effects are known \citep{hill2011bayesian}; we use 10 of the 100
standard simulation realizations released by \citet{shalit2017estimating}. \dataset{Jobs} combines the
randomized LaLonde employment-training experiment \citep{lalonde1986evaluating} with
observational PSID controls; policy value is evaluated on the randomized experimental subset,
and the ten benchmark units are released re-splits of that same sample
\citep{shalit2017estimating}. \dataset{Synthetic} is generated by our own documented data-generating process (DGP;
datasheet in \texttt{docs/DATASETS.md}), and \dataset{Hillstrom}, \dataset{Lenta}, \dataset{X5}, and \dataset{MegaFon} are
public marketing randomized trials with binary outcomes and no counterfactual
\citep{hillstrom2008minethatdata,lenta2022uplift,x52021uplift,megafon2021uplift}.
A \dataset{Criteo} supplement \citep{diemert2018large} ships outside this accounting ($\dagger$ in Table~\ref{tab:datasets}; Appendix~\ref{app:criteo}).

Crucially, \dataset{Synthetic} and \dataset{IHDP}
provide known simulated effect targets ($\pehe$) and \dataset{Jobs} an experimentally
identified policy-value objective (policy risk), while the four marketing RCTs have no
counterfactual. The benchmark's scope and the findings' evidence base are deliberately
distinct: UpliftBench spans 25 primary instances from seven families, but F1 is identified
on the continuous benchmark with known effects --- the 10 main IHDP realizations, replicated
across all 100 --- and F2 is evaluated on Jobs --- the only benchmark whose deployment objective
is experimentally identified.

% ============================================================
\section{Estimators, Metrics, and Protocol}
\label{sec:methods}

\paragraph{Estimators (12).} Five meta-learners (S/T/X/R/DR)
\citep{kunzel2019metalearners,nie2021quasi,kennedy2020optimal}; three class-transformation
models (ClassTrans, TwoModel, SoloModel) \citep{jaskowski2012uplift,kane2014mining}; three
uplift forests (KL/ED/$\chi^2$) \citep{rzepakowski2012decision}; and the causal forest
\citep{wager2018estimation,econml}. Meta-learners and class-transformation models share one
LightGBM base learner \citep{ke2017lightgbm} and HP search space, so differences reflect the
meta-strategy. We later swap in XGBoost to test robustness (Section~\ref{sec:robustness}).

\paragraph{Six metrics.} We compute three \emph{ranking} metrics --- Qini (unnormalized),
AUUC, and uplift-at-$k{=}0.3$; one \emph{policy} metric --- policy value at $k{=}0.3$; and,
where a reference objective exists, $\pehe$ (RMSE of the CATE) and policy risk $R_{\mathrm{policy}}$
\citep{shalit2017estimating}. Calibration is measured by expected calibration error (ECE) \citep{kuleshov2018accurate}.
Each dataset's \emph{primary} metric is what is knowable: $\pehe$ (Synthetic, IHDP), policy
risk (Jobs), Qini (marketing). \textbf{Orientation.} $\pehe$ and $R_{\mathrm{policy}}$ are
error/risk quantities (lower is better); to compare them with the ranking metrics
(higher is better) we always correlate against their negations $-\pehe$ and
$-R_{\mathrm{policy}}$, so a positive correlation means ``tracks the reference objective''
throughout ($\pehe$ is known exactly on IHDP/synthetic; Jobs policy risk is experimentally
estimated). Throughout, ``Qini'' means the canonical cumulative-gain definition below;
library implementations differ in baseline, normalization, and tie conventions, and our
claims are about this canonical object.

\paragraph{Why PolicyValue@$k$ is not in panel (a) of Fig.~\ref{fig:agreement}.} It evaluates one budgeted decision
with treatment-weighted noise, not global effect accuracy ($\rho\approx0.06$ with
$-\pehe$; no oracle reading on continuous data), so it appears only where it is meaningful
(panels (b)--(c), Section~\ref{sec:m2}).

\paragraph{Policy risk and policy value (exact definitions).} These are the two deployment
objectives, and they differ. \emph{Policy risk} follows \citet{shalit2017estimating}: a model
$\hat\tau$ induces the \emph{sign-thresholded} rule $\pi(x)=\mathbb{1}[\hat\tau(x)\ge 0]$, and
on the randomized Jobs experimental subset
\[
  \begin{aligned}
  R_{\mathrm{policy}}(\pi)&=1-\widehat{\mathbb{E}}\big[Y(\pi)\big], \\
  \widehat{\mathbb{E}}\big[Y(\pi)\big]&=\frac1{|\mathcal E|}\!\sum_{i\in\mathcal E}
        \Big[\tfrac{y_i}{\hat\pi_1}\mathbb{1}[\pi_i{=}1,t_i{=}1]\\
  &\qquad\qquad +\tfrac{y_i}{1-\hat\pi_1}\mathbb{1}[\pi_i{=}0,t_i{=}0]\Big].
  \end{aligned}
\]
with $\hat\pi_1$ the empirical treated fraction in the experimental subset $\mathcal E$ (an
inverse-propensity-weighted (IPW) estimate; lower risk is better; no further normalization). \emph{Policy value at $k$} is a
different, budgeted quantity: treat the top-$k$ fraction by score, $V_k=\tfrac1n\big(\sum_{i\in
\text{top-}k,t_i=1} y_i/\hat e_i + \sum_{i\notin\text{top-}k,t_i=0} y_i/(1-\hat e_i)\big)$ at
$k{=}0.3$. Policy risk uses a sign threshold; policy value uses a budget --- we report both.
The two use different propensity conventions by design: for Jobs policy risk we follow the
established benchmark definition of \citet{shalit2017estimating} with the experimental
treated fraction $\hat\pi_1$, whereas PolicyValue@$k$ uses the general per-dataset
implementation across all benchmarks (known propensities on the RCTs, train-fold-estimated
$\hat e_i$ on observational data).

\paragraph{Outer-test-isolated protocol.} All preprocessing, propensity estimation and tuning are
fit within the training fold only; we use stratified 3-fold CV repeated over 3 seeds
(9 evaluations/cell), tuning by random search ($B{=}10$, 2 inner folds). \textbf{The inner
tuning objective is validation Qini on every dataset} (the released default), so the
candidate panel every metric is scored on was itself selected under one of the metrics we
scrutinize; Section~\ref{sec:limitations} states what this does and does not affect.

\paragraph{Uncertainty: what is resampled.} All F1\slash F2 CIs are cluster bootstraps whose resampling unit is the \emph{benchmark realization}, so the effective sample size is the number of realizations (10 continuous, 10 Jobs) --- never the fold rows (Section~\ref{sec:results}). IHDP realizations are independent simulated draws; the 10 Jobs units are \emph{re-splits of one sample}, so their intervals can understate uncertainty, quantified by a design-effect sensitivity in Section~\ref{sec:m2noise}.

\paragraph{A note on Qini normalization.} We report the \emph{unnormalized} Qini throughout and make no claims based on the normalized variant: it divides by a perfect-curve area that is frequently non-positive on continuous outcomes, leaving the score undefined or orientation-reversed there (Appendix~\ref{app:deferred}).

Before turning to results, Table~\ref{tab:artifact} summarizes the reusable benchmark
artifact that operationalizes the analyses below.

% ============================================================
\begin{table*}[t]
\centering
\caption{The UpliftBench artifact at a glance.}
\label{tab:artifact}
\small
\begin{tabular}{ll}
\toprule
Component & UpliftBench implementation \\
\midrule
Dataset interface & Versioned loaders, fixed splits, metadata; user-data adapter \\
Estimator interface & Common fit/predict contract for 12 estimators \\
Metric interface & Ranking, effect-accuracy, policy, and calibration objectives \\
Result schema & Dataset, realization, model, fold, seed, config, git hash, status \\
Reproduction & Smoke run, artifact regeneration, full benchmark target \\
Submission format & Model card plus code or fixed-fold predictions \\
Governance & Versioned releases, independent reruns, visible timeouts \\
\bottomrule
\end{tabular}
\end{table*}

\section{Results}
\label{sec:results}

The benchmark produced 2{,}112 fold-level rows (2{,}052 successful; 60 correctly skipped:
binary-only models on IHDP's continuous outcome). At the dataset $\times$ model level, all 300 cells are accounted for (Criteo supplement outside this ledger; Appendix~\ref{app:criteo}): 228 completed, 60 are inapplicable IHDP/binary-model skips, and
12 timed out: the three uplift forests on each of \dataset{Synthetic}, \dataset{Hillstrom}
and \dataset{MegaFon} (9), DR on \dataset{Lenta}, and R and DR on \dataset{MegaFon}.
The \dataset{Synthetic} forest timeouts at $n{=}2{,}000$ are the fixed per-job budget interacting with the ${\sim}189$ tuning-loop fits per fold evaluation, not a failure at that sample size. Timeouts are marked, not imputed; skips are whole cells (a
completed cell is 9 fold rows).

\subsection{Finding 1: outcome-regime sensitivity of Qini}
\label{sec:m1}

\begin{figure*}[t]
  \centering
  \includegraphics[width=\linewidth]{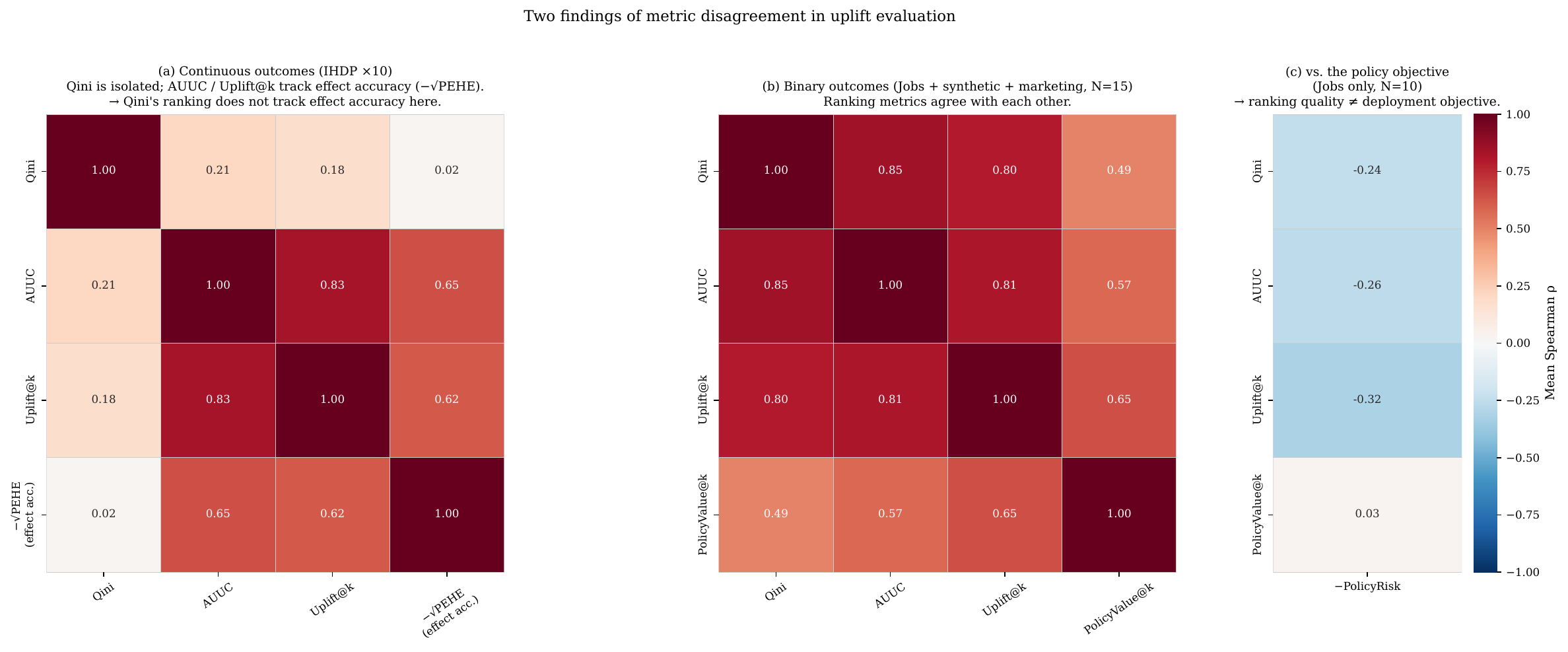}
  \Description{Three heatmap panels of mean cross-metric rank agreement: on continuous outcomes Qini is isolated while AUUC and uplift-at-k track effect accuracy; on binary outcomes the ranking metrics agree with each other; and on Jobs alone every ranking metric diverges from the identified policy objective.}
  \caption{\textbf{Two findings of metric disagreement.} Mean \emph{across dataset
  instances} of the within-instance Spearman $\rho$ computed across models (not across folds).
  Each panel is computed on a single, stated instance set: \textbf{(a)} the 10 continuous
  instances (IHDP $\times$10; \dataset{Synthetic} is generated with a binary outcome and is
  therefore grouped with the binary instances in (b)) --- Qini is isolated, agreeing neither
  with ground truth
  ($-\pehe$, $\rho=-0.15$) nor with sibling ranking metrics AUUC/uplift-at-$k$
  ($\rho\approx0$), which themselves correlate substantially with $-\pehe$
  ($\rho\approx0.6$); \textbf{(b)} the 15 binary instances (10 Jobs $+$ \dataset{Synthetic}
  $+$ 4 marketing) ---
  ranking metrics agree tightly ($\rho\approx0.8$--$0.9$); \textbf{(c)} the 10 Jobs
  splits, the only binary instances where the deployment objective is identified ---
  every ranking metric diverges from it ($\rho\approx-0.25$ with policy value
  $-R_{\mathrm{policy}}$; equivalently, positively correlated with policy \emph{risk}). The
  reference objectives are oriented higher-is-better as $-\pehe$ and $-R_{\mathrm{policy}}$,
  so a positive $\rho$ means ``tracks the objective.''}
  \label{fig:agreement}
\end{figure*}

The failure appears on the evaluated continuous benchmark (IHDP)
(Fig.~\ref{fig:agreement}a). On the primary continuous panel (the 10 IHDP splits and the
six estimators applicable to continuous outcomes --- binary-only models are skipped by
design), the Qini ranking shows no detectable alignment with effect accuracy
($\rho=-0.15$ vs.\ $-\pehe$; $+0.02$ vs.\ AUUC) while AUUC and uplift-at-$k$ correlate
substantially with $-\pehe$ ($+0.68$ and $+0.69$). Tuning off gives a paired gap of $\notuneDelta$ [$\notuneDeltaLo,\notuneDeltaHi$]; inverting the inner objective to AUUC gives $\atDelta$ [$\atDeltaLo,\atDeltaHi$] (Appendix~\ref{app:notune}) --- F1 is an artifact of neither metric's selection. Nor is the six-model rank correlation doing the work: as pairwise concordance over $\concN$ model-pair comparisons across all 100 realizations, Qini orders pairs at $\concQini$ [$\concQiniLo,\concQiniHi$] --- a coin flip --- against AUUC's $\concAuuc$ [$\concAuucLo,\concAuucHi$]. On the
extended common panel --- all $\ihdpN$ IHDP realizations --- Qini's mean rank correlation
with effect accuracy is $\ihdpQini$, 95\% CI $[\ihdpQiniLo,\ihdpQiniHi]$, against a paired
AUUC-over-Qini gap of $\ihdpDelta$ [$\ihdpDeltaLo,\ihdpDeltaHi$]. Within the evaluated
benchmark panel, the observed isolation is specific to Qini rather than a generic
ranking-versus-estimation gap.

\paragraph{F1 is not a tail effect.} The regime that \emph{indexes} F1 is not a mechanism for it: the panel where F1 appears is \emph{mild}-tailed, and on a far heavier-tailed family Qini tracks effect accuracy (Section~\ref{sec:robustness}). In short: \emph{Qini is uninformative on IHDP, not because of heavy tails, and we cannot yet say why}.

\paragraph{Definitions and identities.} Sort units by descending predicted uplift
(deterministic tie-breaking), with cumulative treated/control counts $T_k,C_k$ at depth $k$.
Qini integrates the count-corrected cumulative \emph{sum}
$g(k)=\sum_{i\le k} y_i t_i-\big(\sum_{i\le k} y_i(1-t_i)\big)T_k/C_k$ against a chord
baseline; our AUUC --- the \emph{prefix-mean} AUUC; ``AUUC'' unqualified always means this variant --- integrates the difference of cumulative \emph{means} $u(k)$ minus the ATE triangle (a \emph{shifted} convention: a random ranking scores ${\approx}\mathrm{ATE}/2$, identical for every model on a fold, so it cancels from every reported statistic); both use the population-fraction axis and we report raw (unnormalized) areas (edge cases in Appendix~\ref{app:m1detail}); the shipped \texttt{causalml}\slash\texttt{scikit-uplift} AUUC instead integrates the cumulative-gain curve $k\,u(k)$ --- the \emph{cumulative-gain} AUUC of Section~\ref{sec:robustness}. The two are linked exactly:
\begin{lemma}[Qini is treated-count-weighted AUUC]\label{lem:gk}
On any prefix where both arms are present, $g(k)=T_k\,u(k)$.
\end{lemma}
\begin{lemma}[Score-level relationship]\label{lem:baseline}
With baseline-subtracted residuals $q(k)=g(k)-\tfrac{k}{n}g(n)$ and
$a(k)=u(k)-\tfrac{k}{n}u(n)$,
\[ q(k)=T_k\,a(k)+\tfrac{k}{n}\,u(n)\,(T_k-T_n). \]
\end{lemma}
Thus Qini introduces treated-count depth weighting and an additional treatment-interleaving
term relative to AUUC; these identities provide a structural pathway, not a complete
explanation of the observed IHDP separation (proofs in Appendix~\ref{app:m1detail}).

\paragraph{Scale ruled out; the failure in one construction.} Positive affine
transformations preserve every within-split Qini, AUUC, and uplift-at-$k$ model ranking,
ruling out scale and location as explanations (Prop.~\ref{prop:m1} and empirical
confirmation, Appendix~\ref{app:m1detail}). The distribution's \emph{shape}, however, can
bite: on a hand-checkable 12-unit construction, a single large control outcome ($y{=}40$)
makes Qini prefer a near-random model over a near-true one while AUUC and $-\pehe$ prefer
the better model, and deleting that one outcome flips Qini back
(Table~\ref{tab:cxvectors}).

\paragraph{A controlled test, and an honest boundary.}\label{sec:m1boundary} Holding models and latent CATE fixed and varying only the outcome distribution (Fig.~\ref{fig:transform}), affine transforms change nothing and rising kurtosis degrades \emph{all} cumulative ranking metrics --- Qini and AUUC \emph{together} --- so the controlled experiment does not reproduce Qini's isolation, and candidate mechanisms remain open (Appendix~\ref{app:m1detail}).

The clearest symptom is the \model{DR-Learner} on IHDP: it tops Qini ($\drQiniSeight$ on s8) yet is ranked worst by $\pehe$, which diverges into the tens of thousands (max $\drPeheMax$, one fold of the main run's s3 cell; the per-realization cells of Fig.~\ref{fig:leaderboard} average nine folds and sit lower, s3 cell mean $\approx2.9{\times}10^{4}$) where EconML's augmented-IPW (AIPW) nuisance estimation degenerates --- a known numerical risk of inverse-propensity-weighted pseudo-outcomes under extreme estimated propensities \citep[cf.][]{kennedy2020optimal}. Because F1 is rank-based this magnitude is irrelevant, and F1 survives excluding DR (and DR$+$R) entirely (Appendix~\ref{app:m1detail}).

\subsection{Finding 2: ranking metrics are insufficient for level-dependent deployment objectives}
\label{sec:m2}

Figure~\ref{fig:agreement}(b)--(c) illustrates F2, which has two parts: a \emph{structural}
insufficiency of ranking metrics for level-dependent objectives, and its operational
magnitude on Jobs. On binary outcomes Qini behaves like a normal ranking metric --- it
agrees with AUUC at $\rho=0.90$ (Jobs) and $0.72$ (marketing) --- yet it does not recover
the \emph{policy} objective on Jobs, measured by the RCT-estimated policy risk (an IPW value
on the randomized experimental subset; Section~\ref{sec:methods}).

\begin{proposition}[Rank-invariance boundary]\label{prop:rankinv}
Any evaluation metric that depends on $\hat\tau$ only through the induced ranking of units is
invariant to every strictly increasing transformation of $\hat\tau$, including shifts
$\hat\tau\mapsto\hat\tau+c$. The sign-threshold policy $\pi(x)=\mathbb{1}[\hat\tau(x)\ge 0]$
is not: a shift changes which units are treated. Hence for two models whose scores share a
ranking but differ by a shift, ranking metrics assign identical values while their sign
policies --- and the resulting policy values --- can differ. Ranking metrics therefore cannot,
from rank information alone, identify the sign-threshold-optimal model; that requires the
score \emph{level} (calibration or an explicit threshold).
\end{proposition}

\noindent This is a \emph{non-identifiability} result, not a claim that ranking metrics
carry no information in practice --- real estimators do not differ by pure shifts --- so both
halves are load-bearing: the proposition gives insufficiency in principle, and Jobs supplies
the empirical claim that here the signal is not transmitted. The headline is a selector contrast: \emph{within the released split rotation, direct policy-risk selection lowers Jobs benchmark regret versus random; Qini, AUUC, and uplift-at-$k$ do not.}

\paragraph{On Jobs, the benchmark contains detectable policy-selection signal that the
evaluated ranking metrics do not transmit.} Under a cross-repeat
rotation (select on two repeated-CV seeds, evaluate on the held-out seed; cluster bootstrap
over the 10 Jobs splits), selecting by \emph{cross-fitted policy risk itself} reduced
benchmark risk relative to random selection by $\selRiskGain$ (95\% CI
$[\selRiskGainLo,\selRiskGainHi]$, excluding
zero), whereas the Qini, AUUC, and uplift-at-$k$ selectors changed it by $\selQiniGain$,
$\selAuucGain$, and $\selUpliftGain$ respectively --- none distinguishable from zero
(ladder details in Appendix~\ref{app:extrobust}).

\paragraph{The operational cost: policy-selection regret.} With both the metric winner and
the risk-minimizing reference chosen on the selection seeds and evaluated on the held-out
seed (policy $\pi(x){=}1$ iff $\hat\tau(x){\ge}0$), metric selection incurs mean regret
$\regLo$--$\regHi$ risk --- $\regRelLo$--$\regRelHi\%$ of the best selected-candidate risk
--- and is comparable to the random-selection baseline (paired differences' CIs span zero).
Because repeated-CV seeds re-partition the \emph{same} observations, this is cross-repeat
benchmark-selection regret, not evaluation on a new sample; secondary statistics
(within-margin rates, paired differences per metric) are in Appendix~\ref{app:extrobust}.

\paragraph{Calibrating the threshold closes most of the gap.} As Proposition~\ref{prop:rankinv} predicts, replacing the fixed zero threshold with one calibrated on the selection folds cuts the Qini-selection regret by $\thrClosed\%$ (to $\thrRegCalib$, CI no longer excluding zero): F2 on the sign-threshold objective is largely a score-\emph{level} effect, so pair a ranking metric with a calibrated threshold (Appendix~\ref{app:deferred}).

\paragraph{How strong is this evidence?}\label{sec:m2noise} The Jobs instances re-split one sample and the IPW risk is noisy (SE $\approx\riskUnitSeMed$); the regret interval excludes zero only for between-split dependence $\bar\rho\le\regRhoStarQini$, and the splits' per-model risk vectors already correlate at $\jobsRiskCorr$ across the 10 splits (a raw correlation reflecting both between-model signal and split dependence, not itself $\bar\rho$ of the differenced statistic). We therefore read the regret as \emph{modest} and rest F2 on its convergent parts --- the structural boundary (Prop.~\ref{prop:rankinv}), the selector contrast, and threshold calibration (Appendix~\ref{app:extrobust}).

\begin{center}
\fbox{\parbox{0.92\linewidth}{\centering \textbf{Takeaway (F2).} \emph{Jobs contains usable
policy-selection signal, but the common ranking metrics fail to capture it:} direct
policy-risk selection yields lower benchmark regret than random within the released
split rotation, whereas Qini, AUUC, and uplift-at-$k$ do not; their
selected policies incur a $\regRelLo$--$\regRelHi\%$ regret relative to the risk-selected
candidate. For sign-threshold policies the mismatch is structural --- rank-only metrics
discard the score level --- and \textbf{calibrating the threshold removes $\thrClosed\%$ of
the Qini-selection regret}, which is the paper's most directly actionable recommendation; the gap
disappears when model selection and evaluation are aligned to budgeted policy value.}}
\end{center}

\paragraph{Which model wins depends on regime and metric.} No estimator dominates: simple meta-learners and the causal forest lead the semi-synthetic regimes, simple estimators lead the marketing RCTs (per-dataset winners and critical-difference diagrams in Appendix~\ref{app:deferred}). These standings are protocol-scoped --- samples in the hundreds, 3-fold CV, and a $B{=}10$ tuning budget can under-serve nuisance-heavy learners (DR/R) --- not universal estimator verdicts (Section~\ref{sec:guide}, Step~3).

\subsection{Robustness and boundaries}
\label{sec:robustness}

\subsubsection{Robustness of F1}

\paragraph{Paired, decomposed statistics.} We do not rest on a pooled Qini-vs-reference
correlation: its CI crosses zero, and pooling outcome regimes conflates the two findings (Appendix~\ref{app:extrobust}). The load-bearing statistic is the paired AUUC-over-Qini gap on the continuous panel, $\Delta=\relDelta$ [$\relDeltaLo,\relDeltaHi$], positive in all six exclusion\slash base-learner cells and with its CI excluding zero in five (Table~\ref{tab:m1sens}).

\paragraph{Uncertainty, and all 100 realizations.} On the 10 main splits the paired gap is positive on $\relPos$/10 (Qini $\bTwoQini$ [$\bTwoQiniLo,\bTwoQiniHi$]). Extending to
\emph{all} $\ihdpN$ IHDP realizations (split indices fixed before analysis; meta/forest
learners on the added splits) strengthens the evidence: Qini's correlation with effect
accuracy is $\ihdpQini$ [$\ihdpQiniLo,\ihdpQiniHi$] while AUUC and uplift-at-$k$ track it
at $\ihdpAuuc$ and $\ihdpUplift$; the paired gap is $\ihdpDelta$
[$\ihdpDeltaLo,\ihdpDeltaHi$], positive on $\ihdpDeltaPos$/$\ihdpN$ realizations
(run documentation in Table~\ref{tab:ihdp100doc}). Extending from the 10-split primary panel
to all $\ihdpN$ realizations lowers the point estimate ($\bTwoDelta$ to $\ihdpDelta$) and
tightens the interval, so the primary panel did not understate the gap. With six common
estimators in the extension, each per-realization Spearman is discrete and coarse; inference therefore concerns the mean paired rank statistic across realizations, not fine-grained estimation within any single one.

\paragraph{Estimator exclusion.} Excluding the DR-Learner, and both DR- and R-Learners, keeps Qini's correlation with ground truth indistinguishable from zero in every configuration and the paired $\Delta$ positive in all six (base learner $\times$ exclusion) cells, its CI excluding zero in five (Table~\ref{tab:m1sens}); AUUC's absolute alignment falls as estimators are removed, so we state F1 relatively (AUUC consistently more aligned than Qini), not as a fixed level. $\Delta$ declines ($+0.63\rightarrow+0.40\rightarrow+0.42$; $+0.69\rightarrow+0.44\rightarrow+0.34$) and the decline runs entirely through AUUC's alignment, so part of the advantage is AUUC penalizing the unstable learners' $\pehe$ blow-ups, to which a rank-only statistic is indifferent: \textbf{the advantage attenuates but does not vanish} --- on all $\ihdpN$ realizations the excluded-panel gaps are $\eHundDeltaExDR$ [$\eHundDeltaExDRLo,\eHundDeltaExDRHi$] and $\eHundDeltaExDRR$ [$\eHundDeltaExDRRLo,\eHundDeltaExDRRHi$], CIs excluding zero.

\paragraph{Base learner.} Switching to XGBoost (Fig.~\ref{fig:robust}) leaves the F1
signature intact (Qini$\perp\pehe$, AUUC$\parallel\pehe$) and preserves the ranking
metrics' mutual agreement on binary outcomes. The conclusions are about \emph{metrics},
not a LightGBM artifact.

\paragraph{Implementation variants.} Scoring identical predictions with
\texttt{causalml}'s shipped Qini yields nearly identical rankings (rank $\rho$
$\cmlRankRho$; winner agreement $\cmlWinnerAgree\%$) and the same F1 result. Correcting the R-Learner\slash Causal Forest treatment nuisance to a classifier (an EconML requirement; a regressor through v1.4.0, caught by an external audit) and rerunning the full benchmark leaves F1 intact: corrected gap $\nfixGap$ [$\nfixGapLo,\nfixGapHi$] vs.\ $\nfixOldGap$ [$\nfixOldGapLo,\nfixOldGapHi$] pre-fix, Causal Forest cells unchanged ($\rho=\nfixCfCorr$; Appendix~\ref{app:extrobust}).

\subsubsection{Boundary of F1}

\paragraph{Two boundary-case validation families.} ACIC 2016 \citep{dorie2019automated} supplies $\acicN$ instances with real covariates and known effects. There Qini behaves normally ($\acicQiniMOne$ [$\acicQiniMOneLo,\acicQiniMOneHi$] vs $-\pehe$), AUUC gives $\acicAuucMOne$ [$\acicAuucMOneLo,\acicAuucMOneHi$], and the paired AUUC-over-Qini gap --- the statistic F1 is defined by --- is $\acicGap$ [$\acicGapLo,\acicGapHi$], CI covering zero; its outcome kurtosis (${\approx}0.4$) is similarly mild.

\paragraph{A third continuous family ends the tail explanation.} We add a third family with known effects and a response surface unlike IHDP's (\dataset{Revenue-Synthetic}: lognormal spend, \emph{multiplicative} effect; $\rsN$ realizations, same protocol). \textbf{F1 does not replicate}: Qini tracks effect accuracy ($\rsQini$ [$\rsQiniLo,\rsQiniHi$]) and the paired gap vanishes, $\rsDelta$ [$\rsDeltaLo,\rsDeltaHi$]. The family that fails F1 is the least heavy-tailed of the three (medians in Sec.~\ref{sec:m1}), so heavy tails are neither necessary nor sufficient. F1 appears on one of three evaluated continuous families: Qini is \emph{sometimes} uninformative and \emph{sometimes} the best of the ranking metrics, spanning $+0.02$ to $+0.48$; no measured property yet predicts which case a practitioner is in (Appendix~\ref{app:sepcov}) --- which is why Qini cannot be trusted unvalidated (Appendix~\ref{app:deferred}). This is consistent with \citet{curth2021really}, who argue IHDP is idiosyncratic; we add that the idiosyncrasy is \emph{metric-facing}, and that no proposed property yet predicts the metric-specific gap ex ante --- the one surviving candidate is the lemma-derived composite of Appendix~\ref{app:sepcov}, pending its controlled sweep.

\paragraph{Mechanism probes, summarized.} RATE and oracle residualization improve Qini only modestly, and controlled sweeps of kurtosis, imbalance, and correlated score error do not isolate the separation. A weighting decomposition localizes it: the field's shipped cumulative-gain AUUC (checked against \texttt{causalml} on all $\gainCmlFolds$ folds, $\gainCmlRho$) tracks $-\pehe$ at $\gainAuuc$ [$\gainAuucLo,\gainAuucHi$] on identical predictions over all $\gainAuucN$ realizations --- better than our prefix-mean $\meanAuucRho$ --- while Qini sits at ${\approx}0$: among these functionals the discrepancy localizes to \emph{treated-count}, not depth, weighting (Lemma~\ref{lem:baseline}'s interleaving term). Mechanism otherwise open (values in Appendix~\ref{app:deferred}).

\subsubsection{Robustness and scope of F2}

\paragraph{Base-learner replication.} Under XGBoost, all three ranking metrics correlate
negatively with policy value (Qini $\xregQiniRho$ [$\xregQiniRhoLo,\xregQiniRhoHi$], there
\emph{not} borderline) and the selection regret stays positive ($\xregLo$--$\xregHi$);
Table~\ref{tab:regretaudit} audits both runs side by side.

\paragraph{Budget grid.} The F2 selection regret is positive with the CI excluding zero at
\emph{every} budget $k\in\{0.1,0.2,0.3,0.5\}$ and under the area-under-curve selector;
within-dataset rankings are moderately stable across budgets (mean pairwise Spearman
$\kgridStab$; per-budget values in Appendix~\ref{app:extrobust}).

\paragraph{Objective specificity.} The F2 gap is objective-specific: direct risk selection
captures signal for the sign-threshold policy, while ranking metrics transmit signal when
the reference objective is budgeted policy value. With policy value at a fixed budget as both reference and objective, every selector beats random and metric-vs-reference regrets are indistinguishable from zero at every budget (Appendix~\ref{app:extrobust}) --- which is exactly why the metric must match the intended deployment rule.

\subsection{Calibration: an identification check, not a third finding}
\label{sec:calibration}
Mis-calibration is \emph{not} a third finding: the diff-in-means reliability estimator behind uplift-ECE is unbiased only under randomized assignment, and once the analysis is stratified by identification regime the pooled Qini-vs-ECE disagreement is revealed as an identification artifact --- where ECE is identified, calibration does \emph{not} disagree with Qini. The full stratified analysis, figure, and tables are in Appendix~\ref{app:calibfull}.

\section{Practitioner's Guide: Choose a Metric, Then a Model}\label{sec:guide}
Match the metric to the \emph{outcome type} (on continuous outcomes avoid the unnormalized Qini as the sole criterion; prefer $\pehe$ where identifiable, and compare ranking metrics against each other otherwise) and to the \emph{deployment objective} (evaluate policy value at the operating budget, and assess calibration with an identified estimator when allocation magnitude matters), then validate rankings on more than one dataset. For sign-threshold deployment, pair the ranking metric with a threshold calibrated on held-out selection data (Section~\ref{sec:m2}). On continuous outcomes, bootstrap independent evaluation units (or benchmark realizations where available), compare Qini and AUUC rankings, and treat the winner as unstable where they diverge. The full step-by-step guide is in Appendix~\ref{app:guide}.

% ============================================================
\section{Limitations}
\label{sec:limitations}

\paragraph{No ground truth on marketing data.} We establish F1 on semi-synthetic benchmarks
with known effects, and F2 using
experimentally identified policy value on Jobs; we extrapolate the result only as a caution
for marketing applications; we cannot establish whether Qini misranks on Lenta. The failure
appears on \emph{one of the three} evaluated continuous families with reference effects
(IHDP); across the three the paired gap forms a gradient ($\ihdpDelta$, $\acicGap$, $\rsDelta$)
rather than a binary split; F1 should be read
as a demonstrated failure case on a standard benchmark, not as a property of continuous
outcomes in general.

\paragraph{Protocol limitations.} Three imperfections in the released protocol bound how F1/F2 should be read, and each is disclosed in full in Appendix~\ref{app:deferred}. (i) \emph{The candidate panel is Qini-tuned}, conditioning the candidate set (all metrics score identical out-of-fold predictions); the finding persists with tuning off ($\notuneDelta$) and with the inner objective switched to AUUC ($\atDelta$, Appendix~\ref{app:notune}). (ii) \emph{Propensity nesting is imperfect inside the tuning loop}: on IHDP\slash Jobs the propensity model is fit on the whole outer-training fold and sliced for the inner folds --- a covariate\slash treatment-only deviation (never outcomes) affecting hyperparameter selection only. (iii) \emph{F2's empirical half is within-sample}: selection, calibration, and evaluation draw on the same Jobs observations --- a descriptive case study whose intervals describe benchmark-split variability, not population sampling error. Fixes for (ii) and (iii) are planned follow-up work.

\paragraph{Scale and compute.} Lenta\slash X5\slash MegaFon are capped at 10K by documented stratified subsampling; \dataset{Criteo} (14M) is outside the headline results; a v1.4.3 supplement adds its leaderboard --- the 1M tier subsampled to the released 10K cap for comparability (Appendix~\ref{app:criteo}; binary-regime agreement replicates, Qini--AUUC $+0.74$; enters neither F1 nor F2) --- and a 100K probe: resolution is cap-limited --- 10K gaps are small relative to fold noise and the 10K and 100K orderings did not agree (${\approx}0$ rank correlation); capped leaderboards are subsample-scoped. A handful of uplift-forest and R\slash DR cells timed out (marked, not imputed); none affects F1\slash F2, whose load-bearing cells have no timeouts (correlations complete-case). A full-scale sweep remains versioned future work.

\paragraph{Scope of F1.} We do not claim Qini is categorically invalid on continuous outcomes: the three families form a gradient in the paired gap ($\ihdpDelta$, $\acicGap$, $\rsDelta$), not a binary split, and Qini tracks effect accuracy on both validation families ($\acicQiniMOne$, $\rsQini$). We establish that on the standard continuous benchmark the Qini ranking is effectively unrelated to effect accuracy while AUUC and uplift-at-$k$ stay informative, and report it because practitioners and libraries apply Qini to continuous outcomes and the failure is silent --- and regime-dependent, which is what makes it dangerous.

% ============================================================
\section{Reproducibility}
\label{sec:repro}

The public reproducibility package --- UpliftBench (code, configs, data loaders, and result parquets) --- is released at \url{https://github.com/binshuangli/uplift-bench}; a DOI-bearing archival release will accompany any archival version.
Dedicated analysis targets regenerate all reported tables, figures, and numerical macros
from the committed result artifacts. The core benchmark runs are driven by \texttt{make
repro} (smoke path \texttt{make repro-smoke}); the prediction-level auxiliary analyses (the
causalml variant, threshold calibration, and the F2 robustness checks) are reproduced through
the documented \texttt{make repro-r4pred} target. Every result parquet is stamped with git hash and config; seeds are globally controlled. Settings: 3 folds $\times$ 3 seeds, $B{=}10$ tuning configs, 2 inner folds; base learners LightGBM and XGBoost. Each analysis regenerates from the raw result rows via a dedicated script under \texttt{scripts/} (per-artifact mapping in the repository README; artifact inventory in Table~\ref{tab:artifact}).

\paragraph{Generative-AI assistance.} An AI coding assistant (Claude Code) assisted with code, documentation, manuscript drafting and editing, and brainstorming candidate analyses. The author independently selected the research questions, specifications, methods, and analyses; reviewed, executed, and validated all generated code and outputs against the committed artifacts; and takes responsibility for all results and claims.

% ============================================================
% ============================================================
\section{Benchmark Documentation, Availability, and Ethics}
\label{sec:dnb}

\paragraph{Provenance and curation.} Dataset provenance is specified in
Section~\ref{sec:data} and per-family datasheets ship in \texttt{docs/DATASETS.md}
(source, licence, preprocessing, caveats); no new human-subjects data is
introduced. The two \emph{auxiliary} continuous families (\dataset{Revenue-Synthetic};
\dataset{ACIC 2016}) are regenerated
deterministically and serve as F1's boundary cases (Section~\ref{sec:robustness}). Every
dataset loads through a versioned loader recording size, treatment fraction, outcome
type, and preprocessing; no full-dataset statistic precedes the train/test split.

\paragraph{Availability, licensing, and maintenance.} Code, configs, loaders, result
parquets, and the scripts regenerating every figure and table are MIT-licensed; each
dataset keeps its original license and is fetched by its loader from the original host,
never redistributed (\dataset{Criteo} under Criteo Research's research-use terms).
The repository is public; an archival release with a Zenodo DOI will
accompany the conference version.

\paragraph{Bring your own data.} A documented adapter
(\texttt{UpliftDataset.\allowbreak from\_frame}) and \texttt{register\_loader} run every
estimator and metric on user data under the released protocol.

\paragraph{Leaderboard governance.} Submissions are pull requests adding a model card
plus estimator code or per-fold predictions on the \emph{fixed} released folds\slash seeds
(\texttt{RESULTS.md}, \texttt{CONTRIBUTING.md}); the maintainer re-runs each within the
released budget (single-maintainer; review latency scales accordingly);
results are frozen per version and timed-out runs are shown, not hidden. We do \emph{not}
claim overfitting resistance: folds and outcomes are public --- a reusable-target risk
\citep{kapoor2023leakage}.

\paragraph{Representativeness of the synthetic data.} The continuous synthetic DGPs are
checked against the real semi-synthetic IHDP
data (all $\ihdpN$ realizations) and a controlled kurtosis sweep spanning $0.2$--$70$
(Fig.~\ref{fig:transform}).

\paragraph{Ethical considerations.} All data are public and de-identified by their
publishers; no new personal data is introduced. Mis-specified evaluation misdirects
consequential targeting resources --- the risk this paper reduces.
UpliftBench evaluates neither subgroup fairness nor treatment
burden and does not validate high-stakes deployment.

% (Reproducibility is covered in full by Section~\ref{sec:repro}; no duplicate paragraph
% here -- the track asks that the information be present, not that it be repeated.)

\section{Conclusion}
\label{sec:conclusion}

In uplift evaluation, the metric is the message. On IHDP, AUUC tracks effect
accuracy while Qini does not ($\ihdpDelta$ [$\ihdpDeltaLo,\ihdpDeltaHi$], F1); on Jobs only
policy-risk selection lowers regret (F2), largely removed by threshold calibration. Both
are bounded: F1 holds on one of three continuous families, F2 vanishes under a budgeted
objective. Match the metric to outcome type and
goal: UpliftBench makes that the default.

\bibliography{refs}
\bibliographystyle{ACM-Reference-Format}

\appendix

\section{Deferred discussion and extended material}\label{app:deferred}

\begin{figure*}[t]
  \centering
  \includegraphics[width=0.92\linewidth]{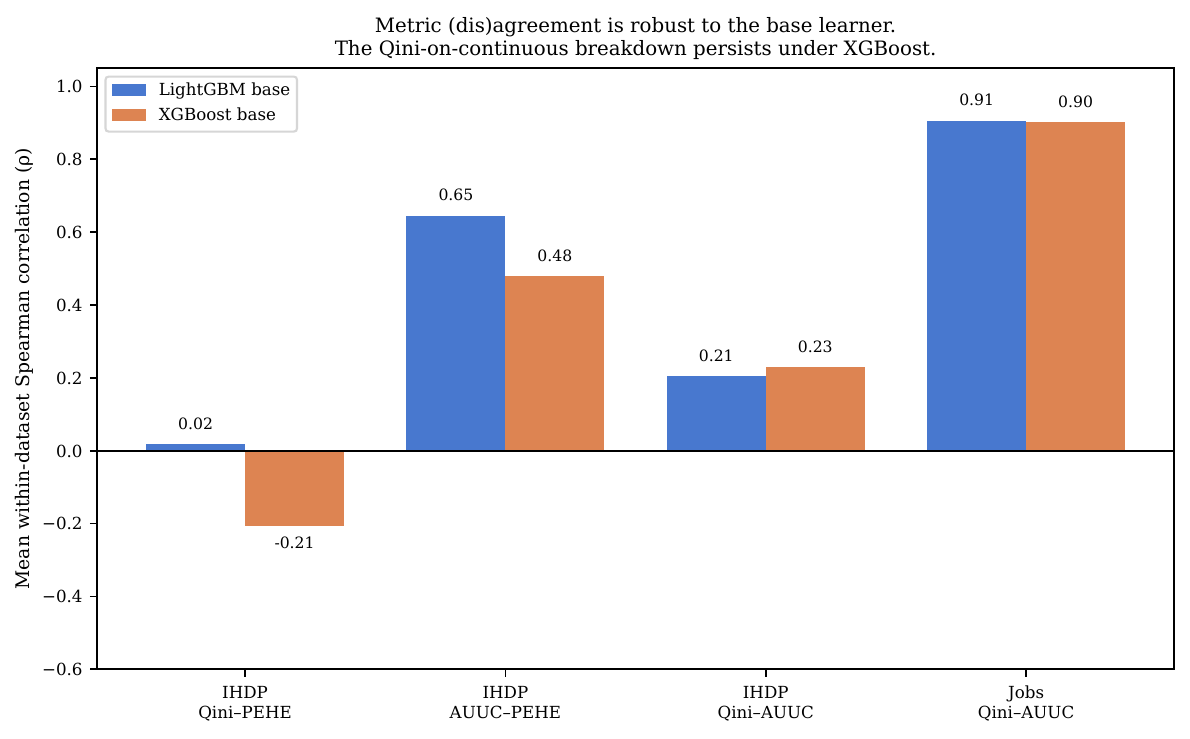}
  \Description{Bar chart comparing metric-versus-reference-objective correlations under LightGBM and XGBoost base learners, showing both mechanisms persist.}
  \caption{The findings are robust to the base learner. Switching LightGBM
  $\rightarrow$ XGBoost, Qini still diverges from effect accuracy ($-\pehe$) on IHDP
  ($\rho:-0.15\rightarrow-0.33$) while AUUC still tracks it ($+0.68\rightarrow+0.53$) --- F1
  --- and ranking metrics still agree with each other on binary
  ($+0.90\rightarrow+0.91$), the mutual-agreement premise of F2. The deployment-objective
  half of F2 is retested under XGBoost in Section~\ref{sec:m2}: the policy-value correlations
  and the cross-repeat selection regret both replicate.}
  \label{fig:robust}
\end{figure*}

\begin{figure*}[t]
  \centering
  \includegraphics[width=\linewidth]{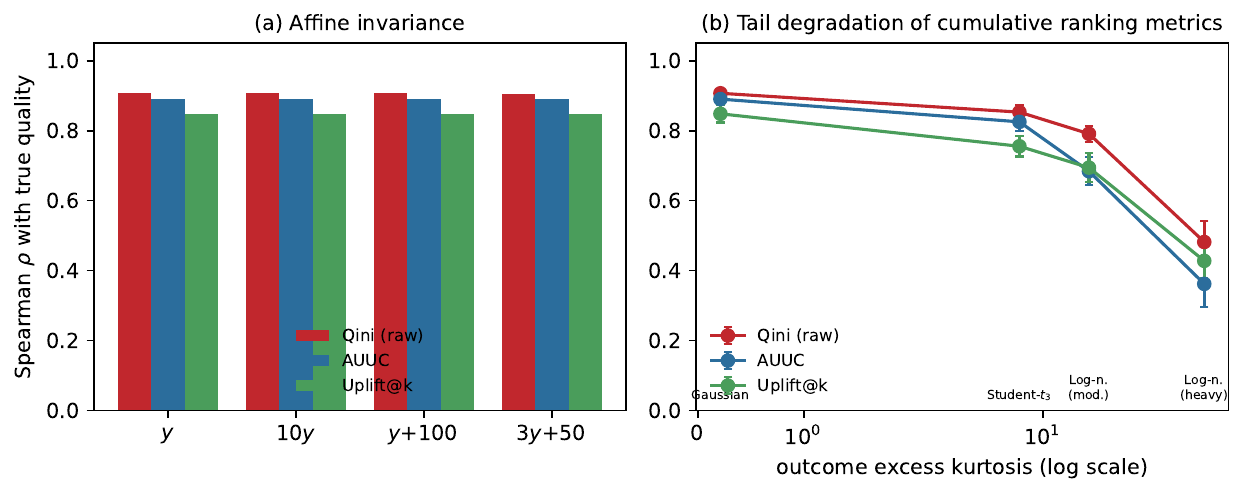}
  \Description{Two panels: metric agreement with the true quality order is invariant to affine outcome transforms, and all cumulative ranking metrics degrade together as outcome kurtosis increases.}
  \caption{\textbf{Controlled probes of the F1 failure signature} (controlled DGP; 8 fixed models of graded
  quality, latent CATE fixed, only the outcome distribution varies; 30 seeds).
  \textbf{(a)} Multiplying or shifting the outcome ($y,\,10y,\,y{+}100,\,3y{+}50$) leaves
  every metric's agreement with the true quality order unchanged --- the failure is not a
  scale artifact (Prop.~\ref{prop:m1}). \textbf{(b)} As outcome excess kurtosis grows, all
  cumulative ranking metrics lose fidelity to true quality, and Qini and AUUC degrade
  \emph{together} --- so the AUUC-over-Qini advantage on the \emph{real} continuous data
  (Fig.~\ref{fig:agreement}a) is a dataset-specific empirical observation, not a guarantee.}
  \label{fig:transform}
\end{figure*}

\paragraph{Which model wins depends on regime and metric (full commentary).} No estimator dominates. By the
appropriate metric, simple meta-learners and the causal forest lead the semi-synthetic
regimes, and simple estimators lead the marketing RCTs (on the released 10K subsamples):
SoloModel on Hillstrom (49.1), ClassTrans on Lenta (20.3), the X- and S-Learner
statistically tied on MegaFon ($\approx$40.3), and on the released 10K X5 subsample no
evaluated model produces a materially positive Qini signal (Qini $\approx$0). These
standings are subsample-scoped, and fold-level standard errors say which are real at the
evaluated $n$: \dataset{Lenta}'s leader shows a large descriptive separation (${>}5$
combined SEs), while the
\dataset{Hillstrom}\slash\dataset{X5}\slash\dataset{MegaFon} leads are within fold noise
(Appendix~\ref{app:criteo}; \texttt{scripts/leaderboard\_resolution.py}). Nemenyi
critical-difference diagrams (reported \emph{descriptively}, not as confirmatory tests) show
wide cliques: on IHDP by $\pehe$, \model{S-Learner} has the best mean rank (2.00, CD$=2.38$)
with DR worst; on binary data by Qini, \model{ClassTrans} leads (2.07, CD$=2.41$). These
standings are protocol-scoped: samples in the hundreds, 3-fold CV, and a $B{=}10$ tuning
budget can under-serve nuisance-heavy learners (DR/R), so ``simple learners lead'' is an
observation about this regime, not a general estimator verdict (Section~\ref{sec:guide},
Step~3).

\paragraph{Calibrating the threshold closes most of the gap.} Proposition~\ref{prop:rankinv}
predicts that the sign-threshold gap should shrink once the score \emph{level} is supplied.
It does: replacing the fixed zero threshold with one \emph{calibrated on the selection folds}
(per model, the threshold minimizing the same out-of-sample IPW policy risk over a 41-point
score-quantile grid on the selection seeds' pooled out-of-fold predictions; the held-out
evaluation seed is untouched) cuts the Qini-selection regret from $\thrRegSign$
[$\thrRegSignLo,\thrRegSignHi$]\footnote{This experiment recomputes its rotation baseline
from the stored per-unit predictions, so it differs in the third decimal from the
selector-ladder rotation's $\regQini$ (Appendix~\ref{app:extrobust}); both are the same
quantity under slightly different rotations.} to $\thrRegCalib$
[$\thrRegCalibLo,\thrRegCalibHi$] --- a
$\thrClosed\%$ reduction, with the calibrated interval no longer excluding zero. So F2 on the
sign-threshold objective is largely a score-\emph{level} effect, exactly as the invariance
boundary implies; the actionable message is to pair a ranking metric with a calibrated
threshold rather than trust $\mathrm{sgn}(\hat\tau)$ (Section~\ref{sec:guide}; full protocol
in Appendix~\ref{app:extrobust}).

\paragraph{Mechanism probes, summarized.} RATE and oracle residualization improve Qini
only modestly ($\rateAutocMOne$ and $\bFourAdj$ vs.\ AUUC's $\auucMOne$; the oracle
adjustment is an optimistic oracle benchmark for outcome-residualization-based variance
reduction; PUC and pROCini were developed for \emph{binary}-outcome uplift evaluation and are
not directly targeted at the continuous-outcome setting studied in F1 --- evaluating
compatible variants remains useful future work --- whereas the variance-reduced Qini of
\citet{bokelmann2024improving} is the highest-priority planned extension because it directly
targets Qini variance), while controlled sweeps of kurtosis, treatment imbalance, and
outcome-correlated score error do not isolate the observed Qini/AUUC separation.
A weighting decomposition sharpens the localization considerably: by
Lemma~\ref{lem:gk}, Qini, the field's shipped cumulative-gain AUUC
(\texttt{causalml}\slash\texttt{scikit-uplift}), and our prefix-mean AUUC are all
integrals $\int w(k)\,u(k)\,dk$ with $w(k)=T_k$, $k\,n$, and $1$ respectively. Scoring
the identical stored predictions on all $\gainAuucN$ IHDP realizations (our
implementation checked against \texttt{causalml.auuc\_score} on all $\gainCmlFolds$
evaluated folds: rank agreement $\gainCmlRho$, within $\gainCmlDev$ of the exact
$n{\times}$ scaling identity): the cumulative-gain AUUC correlates with
$-\pehe$ at $\gainAuuc$ [$\gainAuucLo,\gainAuucHi$] --- \emph{better} than our
prefix-mean AUUC ($\meanAuucRho$ [$\meanAuucRhoLo,\meanAuucRhoHi$]) --- while Qini
sits at ${\approx}0$. Among the evaluated functionals the discrepancy therefore
localizes not to depth weighting but to \emph{treated-count} weighting: $T_k$ differs
from $k\,n\,\pi_1$ exactly by the
treatment-interleaving fluctuation that Lemma~\ref{lem:baseline} isolates and that the
composite $R$ scales with. This also validates the practical advice for the shipped
implementations: cumulative-gain AUUC tracks effect accuracy on IHDP at least as well as
any variant we compute. F1 therefore remains a bounded, benchmark-discovered failure
signature, now localized to treated-count weighting, whose sufficient mechanism is open
(all values in Appendix~\ref{app:extrobust}; \texttt{scripts/gain\_auuc\_check.py}).

\paragraph{Estimator exclusion (full detail).} Recomputing with the DR-Learner excluded, and with both
DR- and R-Learners excluded (Table~\ref{tab:m1sens}), Qini's correlation with ground truth
remains indistinguishable from zero in every configuration (LightGBM: $+0.02 \rightarrow
+0.12 \rightarrow -0.08$; XGBoost: $-0.21 \rightarrow -0.03 \rightarrow -0.12$) and the paired $\Delta$ is
positive in all six (base learner $\times$ exclusion) configurations, with its CI excluding
zero in five of the six (the exception is XGBoost with DR- and R-Learners excluded,
$\Delta=+0.34$ [$-0.16,+0.84$]). The
complement bounds the claim: AUUC's \emph{absolute} alignment falls as unstable estimators
are removed, so we state F1 relatively --- AUUC consistently \emph{more} aligned than Qini
--- not as a fixed absolute level.

We name the pattern rather than leave it to be inferred. $\Delta$ declines
with exclusion under both base learners ($+0.63 \rightarrow +0.40 \rightarrow +0.42$;
$+0.69 \rightarrow +0.44 \rightarrow +0.34$), and the decline is entirely in AUUC's
alignment ($+0.65 \rightarrow +0.52 \rightarrow +0.34$) while Qini's stays flat and near zero.
Part of AUUC's advantage is therefore AUUC correctly penalizing the numerically unstable
learners --- whose fold-level $\pehe$ reaches $\drPeheMax$ (Appendix~\ref{app:m1detail}) ---
to which a rank-only statistic is indifferent. That is a real mechanism, not an artifact, and
it does not exhaust the effect: with both unstable learners removed the LightGBM advantage is
still $+0.42$ [$+0.06,+0.80$] on the 10-split panel, and the wide interval there is a
small-$n$ artifact --- on all $\ihdpN$ realizations the excluded-panel gaps are
$\eHundDeltaExDR$ [$\eHundDeltaExDRLo,\eHundDeltaExDRHi$] (DR removed) and
$\eHundDeltaExDRR$ [$\eHundDeltaExDRRLo,\eHundDeltaExDRRHi$] (DR and R removed), both
CIs comfortably excluding zero. \textbf{The advantage attenuates but does not vanish when the
numerically unstable learners are removed.}

\paragraph{The candidate panel is Qini-tuned.} Every estimator's hyperparameters are selected
by inner-fold validation Qini (Section~\ref{sec:methods}), including on the continuous
benchmarks where we then report that Qini misranks. This conditions the \emph{candidate set},
not the comparison: all six metrics are computed on identical fitted models and identical
out-of-fold predictions, so the paired AUUC-over-Qini gap is a statement about metrics scoring
the same candidates. It does mean the panel is not neutral --- a Qini-tuned search may favour
configurations Qini rates highly. We therefore rerun the primary continuous panel with
tuning switched off entirely (Appendix~\ref{app:notune}): at fixed default hyperparameters,
with no inner Qini search anywhere in the pipeline, the separation persists (paired gap
$\notuneDelta$ [$\notuneDeltaLo,\notuneDeltaHi$]), so F1 is not induced by Qini-tuned
selection. We also run the intermediate case, keeping the released budget but switching the
inner objective from Qini to AUUC: the gap is $\atDelta$ [$\atDeltaLo,\atDeltaHi$], slightly smaller
than as released, so letting the favoured metric choose the panel does not manufacture its
advantage either. Tuning by $\pehe$ or policy risk remains untested and is not a like-for-like
control ($\pehe$ requires the ground truth the metric substitutes for; policy risk targets a
different objective). The bounded budget ($B{=}10$) still compresses candidate spread in both
directions (Appendix~\ref{app:extrobust}).

\paragraph{Propensity nesting is imperfect inside the tuning loop.} For the observational and
semi-synthetic families (IHDP, Jobs), the propensity model is fit once on the whole
outer-training fold and its predictions are then sliced for the inner tuning folds, so
inner-validation rows influenced the propensity values used during inner training. The
deviation involves covariates and treatment assignments only --- the propensity model never
sees outcomes --- and the inner selection criterion (validation Qini) does not use the
propensity at all; it enters only through the fitted candidate. It therefore
affects hyperparameter \emph{selection} only: every reported test-fold metric uses a
propensity model fit without that fold, and the RCT families use their known assignment
probabilities. Refitting propensity strictly within each inner-training fold, with a
regression test enforcing it, is planned for the next release; we do not expect it to move F1
(rank-based, and unchanged under estimator exclusions and base-learner swaps) but we have not
demonstrated that.

\paragraph{F2's empirical half is within-sample.} Selection, threshold calibration, and
evaluation all draw on the same underlying Jobs observations: the cross-repeat rotation
re-partitions one sample rather than holding out fresh units, and the ten released
``instances'' are overlapping re-splits of that sample, so treating them as ten exchangeable
clusters overstates the independent information available (their per-model risk vectors
correlate at $\jobsRiskCorr$; Section~\ref{sec:m2}). The empirical half of F2 is therefore a
\emph{descriptive within-sample case study} of the selection cost, and its intervals describe
benchmark-split variability, not population sampling error; selection and evaluation sharing
outcome noise can also flatter the direct risk selector. The structural half
(Prop.~\ref{prop:rankinv}) does not depend on this design. A genuinely disjoint protocol ---
fit on the training portion, select and calibrate on a validation subset, then estimate policy
value on the untouched experimental test units the loader already exposes, accounting for
repeated unit membership across the ten splits --- is the committed next step.

\paragraph{$\pehe$ is not the sole standard --- by design.} F1 is shown against $\pehe$
\emph{and} sibling ranking metrics, F2 against operational \emph{policy risk}; the conclusions
do not rest on $\pehe$ alone.

\begin{table*}[t]
\centering
\caption{Scope delta against the closest prior work, \citet{yang2026structural}. Both papers
conclude that targeting quality and effect-estimation quality can come apart; the designs
that support the conclusion differ, and the binary-vs-continuous contrast is observable only
in a multi-regime design.}
\label{tab:yangdelta}
\small
\begin{tabular}{lll}
\toprule
Axis & \citet{yang2026structural} & UpliftBench \\
\midrule
Perturbation & injected structural bias & none; observed regimes \\
Families & one semi-synthetic & seven (25 instances) \\
Outcome regimes & continuous only & binary \emph{and} continuous \\
Reference objective & effect accuracy & effect accuracy, policy risk, policy value \\
Question & metric stability\slash robustness & whether metric choice flips conclusions \\
Deployment rule & not varied & sign-threshold vs.\ budgeted \\
Artifact & analysis code & versioned benchmark + result ledger \\
\bottomrule
\end{tabular}
\end{table*}

\paragraph{A third continuous family, and the end of the tail explanation.} Two continuous
families is a two-point sample, so we add a third with known per-unit effects and a response
surface that differs from IHDP's \emph{in kind}: \dataset{Revenue-Synthetic}, in which
baseline spend is lognormal and the treatment effect is \emph{multiplicative}, making the
row-level ITE heavy-tailed by construction --- the setting in which a cumulative-sum
statistic should be most exposed to single large outcomes. We run $\rsN$ realizations with
the same $\rsNModels$ estimators, fold protocol, and tuning budget as the primary panel
(\texttt{make repro-f1-revsynth}; 60 cells, all completing).

\textbf{F1 does not replicate there, and the reason matters.} Qini tracks effect accuracy on
this family ($\rsQini$ [$\rsQiniLo,\rsQiniHi$]) on par with AUUC
($\rsAuuc$ [$\rsAuucLo,\rsAuucHi$]), so the paired AUUC-over-Qini gap vanishes:
$\rsDelta$ [$\rsDeltaLo,\rsDeltaHi$], positive on only $\rsPos$ of $\rsN$ realizations. This
is not a DR-Learner artifact --- excluding it moves the gap to $\rsDeltaExclDR$, and
excluding both DR and R to $\rsDeltaExclDRR$ --- even though the DR-Learner's $\pehe$ diverges
here as it does on IHDP.

Taken with the outcome distributions, this \emph{falsifies} the heavy-tail explanation rather
than merely failing to support it. \dataset{Revenue-Synthetic}'s outcome has median excess
kurtosis $\rsKurt$ (range $\rsKurtLo$--$\rsKurtHi$) and Qini is fine on it; IHDP's primary
panel has median excess kurtosis $\ihdpKurtMed$ --- it is \emph{mild}-tailed, comparable to
ACIC --- and Qini fails on it. The family that breaks F1 is the least heavy-tailed of the
three. (Precisely: excess kurtosis of the \emph{factual outcome}, median $\ihdpKurtMed$ over
the 10 primary IHDP splits. Across all $\ihdpN$ realizations the median is $0.3$ with range
$\kurtLo$--$\kurtHi$, so a minority of IHDP realizations are heavy-tailed while the primary
panel as a whole is not; both statements are used consistently below.) Heavy tails are therefore neither necessary nor sufficient for the failure, which is
consistent with the controlled experiment (Table~\ref{tab:transform}), with the covariate
nulls (Table~\ref{tab:sepcov}), and with nothing else we have tested.

\paragraph{What F1 is, stated at its true scope.} Across three continuous families with
reference effects, Qini's alignment spans $+0.02$ (IHDP) to $+0.48$ (ACIC), with
\dataset{Revenue-Synthetic} at $\rsQini$: it is \emph{sometimes} uninformative and
\emph{sometimes} the best of the ranking metrics, and no measured property yet predicts
which case a practitioner is in --- the lemma-derived composite of
Appendix~\ref{app:sepcov}, whose per-family medians order the three families, is the one
candidate still standing, pending its controlled sweep. That is the operationally important
claim, and it is
weaker than ``Qini fails on continuous outcomes'' while being harder to dismiss: a metric
whose reliability varies unpredictably across datasets cannot be trusted unvalidated on a new
one, which is exactly the diagnostic we recommend (Section~\ref{sec:guide}). F1 names the
failure case and demonstrates that it occurs on a standard, widely used benchmark; it does not
claim the failure is general, and Section~\ref{sec:limitations} records that one of three
evaluated continuous families exhibits it.

\paragraph{A note on Qini normalization.} The optional Qini normalization divides every
model's score within a split by the same perfect-curve area. When that area is strictly
positive, dividing by a common positive constant \emph{preserves} the within-split model
ranking; when it is zero the normalized score is undefined; and when it is negative the
division \emph{reverses} score orientation. On continuous outcomes the standard perfect curve
is built from outcomes treated as counts and its area is frequently non-positive, so the
normalized score is degenerate there. We therefore report the \emph{unnormalized} Qini
throughout and make no claims based on normalized Qini.

\paragraph{ACIC 2016: an independent boundary-case validation family.} ACIC provides an
independent test of the finding's scope: Qini tracks effect accuracy on its continuous
settings, showing that F1 is not a universal continuous-outcome law. Concretely, we add
$\acicN$ instances from the ACIC 2016 competition \citep{dorie2019automated} --- real
Collaborative Perinatal Project covariates ($n{=}4802$), nine heterogeneous-effect settings
stratified over response model $\times$ heterogeneity $\times$ overlap, two replicates
each, generated by the official package (a validation family, not one of the seven primary
benchmark families). ACIC outcomes are continuous with excess kurtosis ${\approx}0.4$ ---
similarly mild (IHDP's primary-panel median is $\ihdpKurtMed$) --- and there Qini
behaves normally: rank correlation with $-\pehe$ is
$\acicQiniMOne$ [$\acicQiniMOneLo,\acicQiniMOneHi$], close to uplift-at-$k$
($\acicUpliftMOne$) and the rank-weighted average treatment effect (RATE-AUTOC,
$\acicRateAutocMOne$); AUUC's is $\acicAuucMOne$
[$\acicAuucMOneLo,\acicAuucMOneHi$], so the paired AUUC-over-Qini gap --- the statistic F1
is defined by --- is $\acicGap$ [$\acicGapLo,\acicGapHi$] on ACIC: mildly positive in
point estimate, with a CI covering zero. The breakdown is observed on the
evaluated IHDP family and not on the evaluated ACIC family; the two differ in many ways,
and Table~\ref{tab:sepcov} shows tail weight does not predict
the metric gap.

\paragraph{A controlled test, and an honest boundary.} Holding models
and latent CATE fixed and varying only the outcome distribution (Fig.~\ref{fig:transform}),
affine transforms change nothing, and rising excess kurtosis degrades \emph{all} cumulative
ranking metrics --- Qini and AUUC \emph{together} --- so the controlled experiment does not
by itself reproduce Qini's isolation. We therefore separate what is mathematically
established (the identities; affine invariance) from the observed signature (Qini's
isolation on the evaluated IHDP benchmark) and from candidate mechanisms, which
remain open after targeted tests (Section~\ref{sec:robustness}; full discussion in
Appendix~\ref{app:m1detail}).

\paragraph{Scope of F1.} We do not claim Qini is categorically invalid on
continuous outcomes --- in a controlled sweep all cumulative ranking metrics degrade together
as tails thicken (Fig.~\ref{fig:transform}), and on both validation families Qini tracks
effect accuracy ($\acicQiniMOne$ on ACIC 2016, $\rsQini$ on
\dataset{Revenue-Synthetic}); the three families form a gradient in the paired gap
($\ihdpDelta$, $\acicGap$, $\rsDelta$) rather than a binary split
(Section~\ref{sec:robustness}). We establish that
on the standard continuous benchmark (IHDP, all $\ihdpN$ realizations) the Qini ranking is
effectively unrelated to effect accuracy while AUUC and uplift-at-$k$ stay informative, not
explained by a single unstable estimator, base learner, Qini implementation variant, or the
tested imbalance-$\times$-tails mechanism. We
report it because practitioners and libraries apply Qini to continuous outcomes and the
failure is silent --- and regime-dependent, which is precisely what makes it dangerous.

Nor is it an artifact of
the six-model rank correlation being coarse: restated as \emph{pairwise concordance} ---
for each realization and each model pair, does the metric order the pair as $-\pehe$ does?
--- Qini agrees with effect accuracy on $\concQini$ [$\concQiniLo,\concQiniHi$] of
the $\concN$ comparisons across all $\ihdpN$ realizations (15 pairs each), indistinguishable from a coin flip,
while AUUC agrees on $\concAuuc$ [$\concAuucLo,\concAuucHi$].

This reading is consistent with
\citet{curth2021really}, who argue that IHDP is idiosyncratic and that conclusions drawn on
it need not generalize: our contribution is to show the idiosyncrasy is also
\emph{metric-facing} --- the field's standard continuous benchmark silently breaks its most
common evaluation metric --- and that no proposed property, theirs or ours, has yet been shown
to predict the metric-specific gap ex ante; the one surviving candidate is the lemma-derived
composite whose per-family medians track the cross-family gradient
(Appendix~\ref{app:sepcov}), pending its controlled sweep.

\paragraph{Uncertainty: what is resampled.} The metric-agreement statistics (F1, F2) are
per-dataset Spearman correlations \emph{across models}, and each cross-metric claim is a mean
of these over datasets. Their CIs are \emph{cluster bootstraps whose resampling unit is the
benchmark realization} (not the fold), keeping all models and metrics within a resampled
realization paired. The two families differ in what a ``realization'' is: IHDP realizations
are conditionally independent \emph{simulated potential-outcome draws} over fixed covariates,
whereas the 10 Jobs units are the train/test \emph{re-splits} of
\citet{shalit2017estimating} over the same underlying LaLonde$+$PSID observations ---
exchangeable partitions of one sample, not independent draws. The cluster bootstrap treats
both as exchangeable clusters; for Jobs this can understate uncertainty to the extent of
between-split dependence, which we quantify with a design-effect sensitivity in
Section~\ref{sec:m2noise}. These intervals quantify variation across realizations of the
benchmark --- not generalization across unrelated real-world datasets --- and the effective
sample size is the number of realizations (10 continuous; 10 Jobs), never the fold
rows (2{,}052; Section~\ref{sec:results}). Within a realization, the model-level metric means already average the
9 folds. We also report
mean pairwise Spearman correlations (complete-case per pair) and Nemenyi critical-difference
diagrams \citep{demsar2006statistical} at $\alpha{=}0.05$ on a comparable model/dataset group.

baseline; our AUUC --- the \emph{prefix-mean} AUUC; ``AUUC'' unqualified always means
this variant --- integrates the difference of cumulative \emph{means} $u(k)$ minus the ATE
triangle --- a \emph{shifted} convention: because $u(k)$ is a prefix mean, a random ranking
scores ${\approx}\mathrm{ATE}/2$ under it, not $0$; the shift depends only on the fold's
data, so it is identical for every model and cancels from every reported statistic ---
and both use the population-fraction axis; we report raw (unnormalized) areas
(edge-case conventions in Appendix~\ref{app:m1detail}). The AUUC shipped by
\texttt{causalml}\slash\texttt{scikit-uplift} instead integrates the cumulative-gain
curve $k\,u(k)$; we call that variant the \emph{cumulative-gain} AUUC and evaluate it in
Section~\ref{sec:robustness}.

All
released results use a classifier for the R-Learner and Causal Forest treatment
nuisance (an EconML requirement for \texttt{discrete\_treatment}); releases up to
v1.4.0 passed a regressor, caught by an external audit, and the full benchmark was
rerun under the corrected nuisance. The fix does not change F1: the corrected panel's
gap is $\nfixGap$ [$\nfixGapLo,\nfixGapHi$] against $\nfixOldGap$
[$\nfixOldGapLo,\nfixOldGapHi$] for the archived pre-fix panel, with the Causal
Forest cells essentially unchanged across the fix (cell-level $\rho=\nfixCfCorr$);
full comparisons in Appendix~\ref{app:extrobust}.

The \dataset{Synthetic} forest timeouts at $n{=}2{,}000$ reflect the per-job wall-clock
budget interacting with the $B{=}10\times2$ inner tuning loop (${\sim}189$ forest fits
per fold evaluation), not a failure at that sample size; the budget, like everything
else, is fixed and released.

\section{Full practitioner's guide}\label{app:guide}

\paragraph{Continuous-outcome diagnostic.} Do not use Qini as the sole selector. Bootstrap
independent evaluation units (or benchmark realizations where available; folds re-partition
the same sample and are not the inferential unit), compare Qini and AUUC model rankings, and
report their rank correlation and winner agreement. When the rankings diverge materially, prefer an
effect-accuracy or deployment-aligned objective where identifiable and treat the winner
as unstable.

\begin{enumerate}[leftmargin=*,label=\textbf{Step \arabic*.}]
\item \textbf{Match the metric to the outcome type.} For \emph{continuous} outcomes, do not use
  the unnormalized Qini as the \emph{sole} model-selection criterion without
  validating it against an effect-accuracy, policy, or average-uplift criterion. In our IHDP
  benchmark it fails to track ground-truth quality while AUUC and uplift-at-$k$
  perform substantially better (F1) --- but on \dataset{Revenue-Synthetic} the two are on
  par ($\rsDelta$
  [$\rsDeltaLo,\rsDeltaHi$]), so \emph{no standing preference between the ranking metrics is
  warranted}. Prefer $\pehe$ where effects are known or simulable; otherwise compare the
  ranking metrics against each other on your own data and distrust the winner wherever they
  diverge. For binary outcomes, Qini can be used as a
  ranking metric, subject to the usual variance, tie, and objective-alignment checks.
\item \textbf{Match the metric to the objective.} For \emph{ranked targeting} at a fixed
  cutoff, a ranking metric suffices. For \emph{sign-threshold deployment}
  ($\pi(x)=\mathbb{1}[\hat\tau(x)\ge0]$), a ranking metric alone is structurally insufficient
  (Prop.~\ref{prop:rankinv}): pair it with a decision threshold calibrated on held-out
  selection data, which closed most of the observed selection gap on Jobs
  (Section~\ref{sec:m2}). For \emph{budget-constrained allocation}, evaluate policy
  value directly at the relevant budget (F2); if allocation magnitude depends on predicted
  effect size rather than rank alone, additionally assess calibration with an identified
  estimator appropriate to the treatment-assignment regime
  (Section~\ref{sec:calibration}).
\item \textbf{Respect the regime.} In our randomized, small-fold benchmark settings,
  R/DR-learners sometimes became numerically unstable while simpler learners (ClassTrans on
  binary outcomes, S-Learner) were competitive. Include simple baselines, and reach for
  orthogonal/doubly-robust learners when nuisance estimates can be supported reliably ---
  typically observational data with adequate sample size.
\item \textbf{Validate across datasets.} Single-dataset rankings transfer only moderately
  (within-regime reference-objective $\rho\approx0.5$); confirm on more than one dataset.
\end{enumerate}

\section{Calibration: the full identification-stratified analysis}\label{app:calibfull}

% Each panel gets its own full-width float: side-by-side subfigures downscaled these
% ~4.6x/5.5x in the two-column layout, which made the tick labels unreadable.
\begin{figure*}[t]
  \centering
  \includegraphics[width=\linewidth]{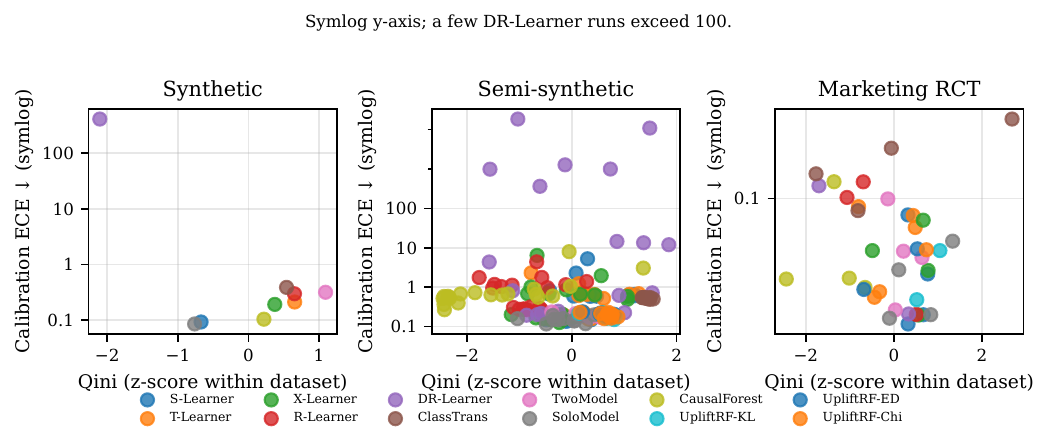}
  \Description{Scatter of within-dataset Qini z-score against calibration ECE by regime, shown descriptively.}
  \caption{Qini (z-scored within dataset) vs.\ calibration ECE, by regime, shown
  \emph{descriptively} (no pooled inferential fit; observations are clustered within
  instances). Only the randomized regimes identify ECE (see text).}
  \label{fig:ece}
\end{figure*}

\begin{figure*}[t]
  \centering
  \includegraphics[width=\linewidth]{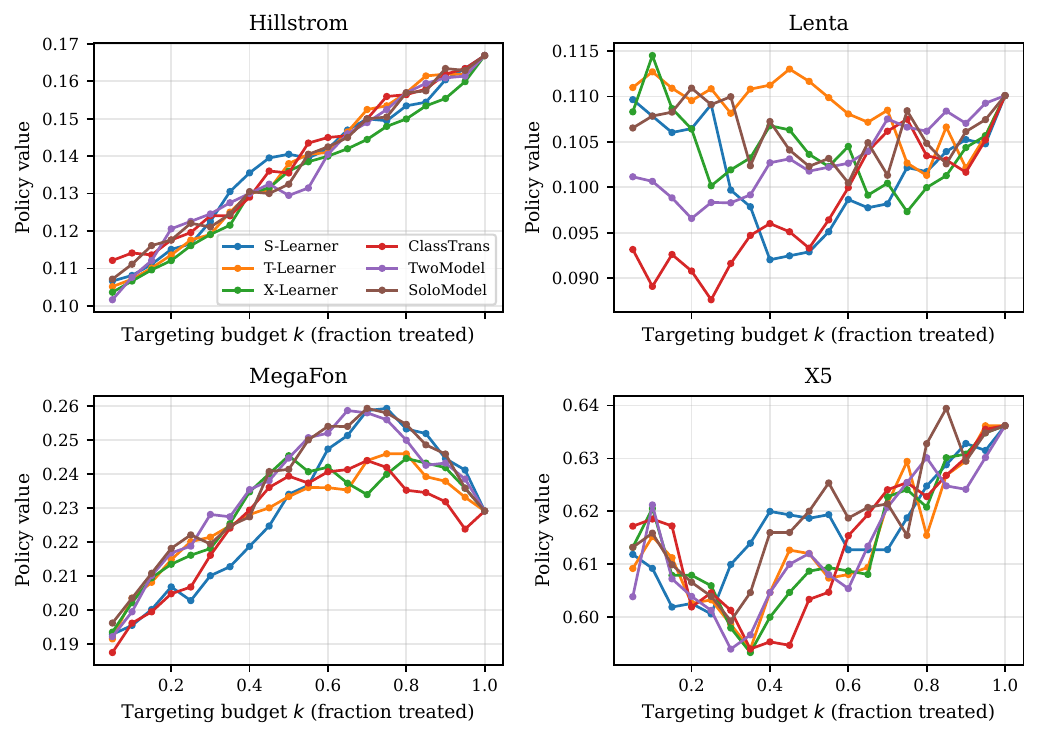}
  \Description{Line chart of estimated policy value against targeting budget for the marketing datasets; the budget-optimal model changes with the budget.}
  \caption{Policy value vs.\ targeting budget $k$ (marketing RCTs). The budget-optimal
  model changes with $k$ --- e.g.\ on Hillstrom the best model shifts from ClassTrans at
  $k{=}0.1$ to SoloModel at $k{=}0.9$ --- so no single ranking fixes the deployment choice.}
  \label{fig:pv}
\end{figure*}

\paragraph{Diagnostic and its identification requirement.} We measure uplift calibration by a
bin-based diagnostic analogous to expected calibration error (ECE): on the held-out test fold,
units are grouped into 10 equal-count (quantile) bins of predicted uplift, and we report the
bin-size-weighted mean absolute gap between the mean predicted uplift and the observed
within-bin uplift, excluding bins that lack a treated or a control unit. The within-bin
treated-minus-control difference identifies the bin-average uplift \emph{only under randomized
treatment}. This holds on the marketing RCTs and the synthetic RCT; it does \emph{not} hold on
IHDP (a confounded semi-synthetic design) or on Jobs (which mixes an experimental arm with
observational PSID controls), where the contrast conflates uplift with selection. Calibration
of treatment-effect predictions is a recognized concern with dedicated estimators
\citep{leng2024calibration,vanderlaan2023isotonic,yadlowsky2025rate}; ours is a coarse
instance, valid only where treatment is randomized.

\paragraph{Where it is identified, calibration does \emph{not} disagree with Qini.} On the
$\calNident$ randomized instances the within-instance Spearman correlation between Qini and
$-$ECE is \emph{positive} ($\calRhoId$, $95\%$ cluster-bootstrap CI
$[\calRhoIdLo, \calRhoIdHi]$), and the paired ECE cost of selecting by Qini is small (median
$\calDeltaMed$, CI $[\calDeltaLo, \calDeltaHi]$). Mis-calibration is therefore \emph{not} a
third instance of F2: where it can be measured without confounding, calibration is moderately
\emph{aligned} with Qini and the ECE cost of selecting by Qini is small in the median,
though imprecisely bounded (the interval's upper end exceeds most binary-outcome ECEs). (Pooling over all $25$ instances
instead drives the correlation to $\calRhoAll$ --- driven by applying an unidentified
within-bin contrast to the confounded IHDP/Jobs data; the pooled contrast and a bin-count robustness check are in
Appendix~\ref{app:calibration}.) An identified analysis on the semi-synthetic datasets would
need oracle-CATE targets and stored per-unit predictions, which we leave to the release. The
separate observation that the budget-optimal model \emph{changes with the budget}
(Fig.~\ref{fig:pv}) still holds and reinforces F2: there is no single ranking a practitioner can
read off in advance.

% ============================================================

\section{F1: proofs, derivations, and detailed constructions}
\label{app:m1detail}

\paragraph{Edge-case conventions (as shipped).} The curves are anchored at $(0,0)$; on any
prefix with $C_k=0$ the gain contribution is set to $0$, and where $T_k=0$ or $C_k=0$ the
uplift value $u(k)$ is treated as $0$ (undefined prefixes do not contribute), so both
integrals begin effectively once each arm appears. Ties are broken by input order,
deterministically.

\begin{proof}[Proof of Lemma~\ref{lem:gk}]
\begin{align*}
T_k\,u(k)&=T_k\big(\tfrac1{T_k}\!\sum_{i\le k} y_i t_i-\tfrac1{C_k}\!\sum_{i\le k}y_i(1-t_i)\big)\\
&=\sum_{i\le k}y_i t_i-\tfrac{T_k}{C_k}\sum_{i\le k}y_i(1-t_i)=g(k). \qedhere
\end{align*}
\end{proof}

\begin{proof}[Proof of Lemma~\ref{lem:baseline}]
By Lemma~\ref{lem:gk}, $q(k)=T_k u(k)-\tfrac{k}{n}T_n u(n)
=T_k\big[u(k)-\tfrac{k}{n}u(n)\big]+\tfrac{k}{n}u(n)(T_k-T_n)
=T_k\,a(k)+\tfrac{k}{n}u(n)(T_k-T_n)$.
\end{proof}

Two effects therefore separate the scores: (i) Qini depth-weights the uplift contrast by
$T_k$; and (ii) an additional term proportional to the deviation of the cumulative treated
count $T_k$ from its depth-proportional value $\tfrac{k}{n}T_n$, which depends on how a
model's ranking interleaves treated and control units. Continuous outcome magnitudes and the
local treated/control composition thus enter the two integrated scores differently.

\begin{proposition}[Affine invariance of within-split rankings]\label{prop:m1}
Fix a split and replace each outcome $y_i$ by $\alpha y_i+\beta$ with $\alpha>0$. Then
every model's $Q$, $A$, and uplift-at-$k$ is multiplied by $\alpha$ (the additive $\beta$
cancels through the treated/control count correction), so the induced ranking of models is
unchanged.
\end{proposition}

\begin{proof}
Shift: $g(k)\!\to\! g(k)+\beta\big[T_k-C_k\,(T_k/C_k)\big]=g(k)$, and
$u(k)\!\to\! u(k)+\beta(1-1)=u(k)$; both areas are shift-invariant. Scale multiplies every
outcome, hence $g,u$ and their areas, by $\alpha$. Within a split all models share the same
$(\alpha,\beta)$, and a common positive factor preserves order. \qedhere
\end{proof}

Table~\ref{tab:affine} confirms this empirically: multiplying or shifting the outcome leaves
every metric's agreement with the true quality order unchanged to two decimals. The F1
failure is therefore not driven by outcome scale or location; distributional shape remains
relevant but does not by itself explain the observed Qini/AUUC separation.

\paragraph{The 12-unit counterexample in full.} On $n{=}12$ units with balanced treatment
and one large control outcome ($y{=}40$), a near-random model~A and a near-true model~B
satisfy: Qini prefers the \emph{worse} A ($Q_A{=}\cxQiniA$ vs.\ $Q_B{=}\cxQiniB$), while
both AUUC ($\cxAuucA$ vs.\ $\cxAuucB$) and $-\pehe$ ($\pehe$ $\cxPeheA$ vs.\ $\cxPeheB$)
prefer B. Deleting that one outcome flips Qini to prefer B ($\cxQiniAdrop$ vs.\
$\cxQiniBdrop$), establishing the reversal is caused by outcome magnitude --- not treatment
imbalance or estimator instability. Table~\ref{tab:cxvectors} lists all twelve units
(treatment, outcome, true CATE, both models' scores and induced ranks) so the example is
verifiable by hand.

\paragraph{The three-level claim taxonomy in full.} \emph{Mathematically established:}
rankings are affine-invariant (Prop.~\ref{prop:m1}) and $Q$ integrates a depth-weighted
version of AUUC's contrast (Lemma~\ref{lem:gk}). \emph{Observed signature:} on the
continuous benchmark family we evaluate (the IHDP realizations; the shipped
\dataset{Synthetic} generator is binary-outcome and belongs to the binary panel),
Qini's ranking --- unlike AUUC's --- fails to track effect accuracy, and
a single large outcome can reverse Qini against it. \emph{Candidate explanation, not yet
isolated:} exactly why AUUC remains informative there while Qini does not; the controlled
experiment shows heavy tails degrade both together, so we do not claim heavy tails alone
explain the separation, nor that AUUC is universally tail-robust --- its advantage on these
data is a dataset-specific empirical observation. This is why F1 is named
\emph{outcome-regime sensitivity} rather than a proven binary-vs-continuous theorem.

\paragraph{DR-Learner detail and the estimator-panel caveat.} The DR-Learner's nested
nuisance estimation diverges on small folds (median per-realization $\pehe\approx24$ on
IHDP --- with realization means reaching $\approx2.9{\times}10^{4}$ --- vs.\ $\approx0.9$
for T-/S-Learner; on synthetic it degenerates to $\pehe\approx721$ while the CausalForest
and S-Learner lead at $\approx0.09$). Metric agreement is inherently evaluated
over the candidate estimators, so a very different panel could yield different correlations;
we claim robustness to exclusions and base-learner swaps, not estimator-panel independence.

\paragraph{DR-Learner instability.} The DR-Learner is EconML's
\texttt{DRLearner} with default nuisance settings (default \texttt{min\_propensity}; no
additional clipping). On IHDP's ${\approx}450$-unit folds the propensity model can predict
near $0$ or $1$, so the inverse-weighted AIPW pseudo-outcome --- and hence $\pehe$ ---
diverges on particular seeds/folds (up to $\drPeheMax$) --- a known numerical risk of
inverse-propensity-weighted pseudo-outcomes under extreme estimated propensities
\citep[cf.][]{kennedy2020optimal}. It is seed/fold-dependent rather than a tuning-budget
artifact: raising the budget from $B{=}10$ to $B{=}50$ at fixed seed does not remove it
(split-0 mean $\drAuditBTen\rightarrow\drAuditBFifty$ over three folds --- a small audit,
so read as \emph{not explained by} the budget rather than as invariance). Because the F1
analysis is rank-based, this magnitude does not affect any reported correlation, and the
paired AUUC-over-Qini gap holds in all six estimator-exclusion $\times$ base-learner
configurations.

\section{Extended robustness and uncertainty analyses}\label{app:extrobust}

\begin{figure}[ht]
  \centering
  \includegraphics[width=\linewidth]{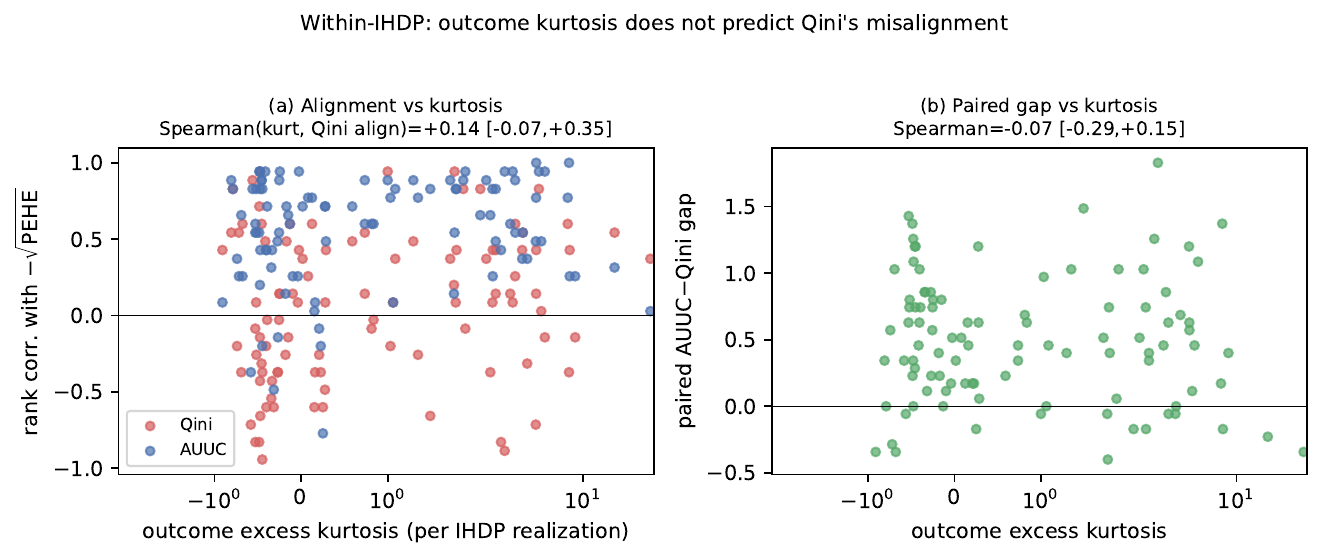}
  \Description{Figure; see the caption for details.}
  \caption{\textbf{Within IHDP, outcome kurtosis does not predict Qini's misalignment}
  (100 realizations). (a) Per-realization rank correlation of Qini and AUUC with
  $-\pehe$ vs.\ outcome excess kurtosis; (b) the paired AUUC$-$Qini gap vs.\ kurtosis.
  Spearman(kurtosis, Qini alignment) $=\kurtQiniCorr$ [$\kurtQiniCorrLo,\kurtQiniCorrHi$];
  Spearman(kurtosis, gap) $=\kurtGapCorr$ [$\kurtGapCorrLo,\kurtGapCorrHi$] --- weakly
  positive and weakly negative respectively, neither useful as a separating diagnostic, so
  the F1 signature is not a within-family tail effect and no scalar kurtosis threshold
  separates failing from non-failing realizations.}
  \label{fig:ihdpkurt}
\end{figure}

\subsection{Searching for a covariate that separates the realizations}
\label{app:sepcov}

Section~\ref{sec:robustness} reports that per-realization outcome kurtosis does not predict
Qini's misalignment within IHDP. Kurtosis is one summary among many, and it is a natural
objection that it is simply the wrong one: the realizations differ by more than an order of
magnitude in outcome \emph{level}, and the $\pehe$ blow-ups concentrate on particular splits.
We therefore repeat the search with three further covariates, all computable ex ante from
observed data with no ground-truth effects: the mean outcome, its coefficient of variation,
and the ratio of mean outcomes between arms.

Two points of interpretation matter before the numbers. First, by Prop.~\ref{prop:m1} every
metric's within-split ranking is invariant to an affine transform of the outcome, so outcome
level or scale \emph{per se} cannot drive F1 --- rescaling a dataset changes nothing. These
covariates vary across realizations because the response surfaces differ, so any correlation
identifies a property the covariate \emph{proxies}, not an effect of scale. The between-arm
ratio is the most interesting candidate precisely because it is arm-asymmetric and therefore
not affine-invariant. Second, F1 is a claim about one metric \emph{relative to another}: a
covariate that predicts how well both Qini and AUUC align merely marks a harder realization,
whereas only a covariate predicting the AUUC-over-Qini \emph{gap} would be a candidate
mechanism. Table~\ref{tab:sepcov} therefore reports all three.

The result is a clean negative on the mechanism question and a modest positive on
difficulty. Mean outcome and the arm ratio both predict alignment for both metrics --- higher
outcome levels go with better alignment, larger between-arm ratios with worse --- so how hard
a realization is for cumulative metrics is partly foreseeable. But no covariate predicts the
gap: the arm ratio, the strongest single-metric correlate at $\sepcovarmratioQini$
[$\sepcovarmratioQiniLo,\sepcovarmratioQiniHi$] against Qini's alignment, moves the gap by
$\sepcovarmratioGap$ [$\sepcovarmratioGapLo,\sepcovarmratioGapHi$]. This mirrors the
controlled tail experiment (Table~\ref{tab:transform}), where thickening the tail degrades
Qini and AUUC together. With $4\times3$ correlations tested, marginal intervals deserve
scepticism; the load-bearing reading is the null on the gap, which is the quantity F1 is
about.

\paragraph{A pre-specified, lemma-derived composite.} Unlike the exploratory covariates
above, one predictor was fixed --- direction and specification --- before being computed.
Lemma~\ref{lem:baseline}'s interleaving term is scaled by the full-sample ATE and its noise
amplified by treatment imbalance, implying that
$R = |\mathrm{ATE}|/\mathrm{SD}(\tau)\times\sqrt{(1-\pi_1)/\pi_1}$ (oracle
$\mathrm{SD}(\tau)$, used to identify, not to recommend) should predict worse Qini
alignment and a larger gap. Tested within family, the first half \emph{confirms in both}:
$\mathrm{corr}(R,\rho_{\mathrm{Qini}})$ is $\mechQiniIhdp$
[$\mechQiniIhdpLo,\mechQiniIhdpHi$] across the 100 IHDP realizations and
$\mechQiniAcic$ [$\mechQiniAcicLo,\mechQiniAcicHi$] across the 18 ACIC instances --- the
first replicated ex-ante predictor of anything in this paper. But the gap half is null in
both ($\mechGapIhdp$ [$\mechGapIhdpLo,\mechGapIhdpHi$]; $\mechGapAcic$
[$\mechGapAcicLo,\mechGapAcicHi$]; the ACIC gap CI is wide enough that 18 instances are
simply underpowered for the gap test), and on IHDP $R$ predicts AUUC's alignment about as
strongly ($\mechAuucIhdp$ [$\mechAuucIhdpLo,\mechAuucIhdpHi$]; on ACIC the AUUC half is
directionally weaker, $\mechAuucAcic$ [$\mechAuucAcicLo,\mechAuucAcicHi$], but its CI
covers zero at $n{=}18$). So \emph{within} a family, by this appendix's own discipline, $R$
is a \emph{difficulty} covariate: it says when cumulative metrics degrade, not why Qini
alone decouples.

\emph{Across} families, however, $R$ does what no other quantity in this paper has: its
per-family medians order the three continuous families exactly as the observed gap gradient
does --- IHDP $\rMedIhdp$, ACIC $\rMedAcic$, \dataset{Revenue-Synthetic} $\rMedRev$
against gaps $\ihdpDelta$, $\acicGap$, $\rsDelta$. Three ordered medians is, by itself, a
one-in-six chance event, so we report this as a \emph{candidate} family-level separating
property, not a mechanism; and unlike the direction and specification of $R$, which were
fixed before computation, the family-level comparison was specified \emph{after} the
gradient was observed, so it carries no pre-registration credit. But it is the first
candidate to survive a test, and its \emph{proxy} form --- $|\widehat{\mathrm{ATE}}|$ over
the spread of predicted CATEs, pooled across candidates --- is genuinely ex-ante computable
and tracks the oracle where estimates are stable: restricted to the four numerically stable
learners, oracle-vs-proxy Spearman is $\proxTrackFourIhdp$ on IHDP and
$\proxTrackFourAcic$ on ACIC (proxy medians $\proxMedFourIhdp$, $\proxMedFourAcic$
against oracles $\rMedIhdp$, $\rMedAcic$). Pooled over all six candidates it fails on IHDP
($\proxTrackSixIhdp$): DR\slash R prediction blow-ups dominate the predicted-CATE spread
--- consistent with this paper's DR account, and a caution that the proxy is usable only
alongside an estimate-stability check. All of this sharply motivates the controlled sweep of
$R$ (the imbalance $\times$ effect-to-heterogeneity cell our five sweeps did not visit),
which is committed next-version work.
\texttt{python scripts/f1\_separating\_covariates.py} and
\texttt{scripts/r19\_analyses.py} regenerate this appendix.

\begin{table*}[t]
\centering
\caption{Searching for a realization-level covariate that separates the IHDP realizations where Qini's ranking tracks effect accuracy from those where it does not. Spearman correlation across the $\sepcovN$ realizations between each ex-ante covariate (computable from observed data alone) and two alignment outcomes plus AUUC's alignment, with 95\% paired bootstrap CIs. Only a covariate predicting the \emph{gap} would be a candidate mechanism for F1; covariates predicting both alignments merely mark harder realizations. By Prop.~\ref{prop:m1} an affine transform of the outcome leaves every ranking unchanged, so outcome level or scale cannot itself drive F1; these covariates vary because the response surfaces differ, and the arm ratio is the one candidate that is not affine-invariant.}
\label{tab:sepcov}
\small
\begin{tabular}{llrrr}
\toprule
Covariate & Range & vs.\ Qini align. & vs.\ AUUC align. & vs.\ gap \\
\midrule
excess kurtosis & -0.9--27.0 & $+0.14$ [$-0.07,+0.35$] & $+0.13$ [$-0.08,+0.33$] & $-0.07$ [$-0.29,+0.15$] \\
mean outcome $\bar y$ & 1.7--60.9 & $+0.24$ [$+0.05,+0.42$] & $+0.34$ [$+0.14,+0.51$] & $+0.03$ [$-0.18,+0.23$] \\
coef.\ of variation & 0.2--1.1 & $-0.32$ [$-0.49,-0.13$] & $-0.37$ [$-0.55,-0.17$] & $-0.01$ [$-0.21,+0.19$] \\
arm outcome ratio & 0.9--4.9 & $-0.33$ [$-0.49,-0.14$] & $-0.32$ [$-0.50,-0.12$] & $+0.05$ [$-0.15,+0.25$] \\
\bottomrule
\end{tabular}
\end{table*}

\subsection{F1 without tuning: is the separation an artifact of the Qini-tuned panel?}
\label{app:notune}

Because the released protocol tunes every estimator by inner-fold validation Qini
(Section~\ref{sec:limitations}), the candidate panel that all metrics score was itself
selected under one of the metrics we scrutinize. The sharpest test of whether this induces
F1 is to remove tuning altogether. We rerun the primary continuous panel --- the same 10
IHDP splits, the same $\notuneNModels$ estimators applicable to continuous outcomes, the
same repeated stratified $3\times3$ fold protocol --- with \texttt{tuning.enabled=false}, so
every estimator runs at its library default hyperparameters and no inner Qini search occurs
anywhere in the pipeline (\texttt{make repro-f1-notune}; 60 cells, all completing).

The separation persists. Rank correlation with $-\pehe$ across models, averaged over the 10
splits with a dataset-level bootstrap CI: Qini $\notuneQini$
[$\notuneQiniLo,\notuneQiniHi$] --- again no detectable alignment --- against AUUC
$\notuneAuuc$ [$\notuneAuucLo,\notuneAuucHi$] and uplift-at-$k$ $\notuneUk$
[$\notuneUkLo,\notuneUkHi$]. The load-bearing paired AUUC-over-Qini gap is $\notuneDelta$
[$\notuneDeltaLo,\notuneDeltaHi$], positive on $\notunePos$ of 10 splits, with the CI
excluding zero. This is the same qualitative picture as the tuned panel and a gap of
comparable magnitude, obtained with the Qini objective removed from model selection
entirely. F1 is therefore not an artifact of Qini-tuned hyperparameter search: it is
reproduced when nothing in the pipeline optimizes Qini.

\paragraph{The other tuning arm: selecting the panel by AUUC} Removing tuning is not a
neutral control either, because library defaults are not equally favourable to all six
estimators. The complementary test is to keep tuning at the released budget and
\emph{switch} the objective, selecting every candidate by inner-fold validation AUUC --- the
metric F1 says is the reliable one on these data. If the AUUC-over-Qini gap were an artifact
of the selection criterion, this arm is where it should be largest; the mirror-image
objection, that AUUC only looks good because we never let it choose the panel, is settled
here too (\texttt{make repro-f1-auuctune}; 60 cells, all completing).

The gap does not grow. With the inner objective set to AUUC, Qini's rank correlation with
$-\pehe$ is $\atQini$ [$\atQiniLo,\atQiniHi$] --- still no detectable alignment --- against
AUUC's $\atAuuc$ [$\atAuucLo,\atAuucHi$], for a paired gap of $\atDelta$
[$\atDeltaLo,\atDeltaHi$], positive on $\atPos$ of 10 splits. That is slightly
\emph{smaller} than the released Qini-tuned panel's $\relDelta$ [$\relDeltaLo,\relDeltaHi$] and close to the
untuned arm's $\notuneDelta$. Across all three arms --- panel selected by Qini, by nothing,
and by AUUC --- Qini's correlation with effect accuracy is indistinguishable from zero and
AUUC's is substantial, so the finding is invariant to which of the two metrics does the
selecting. What remains untested is tuning by $\pehe$ or by policy risk; the first is
unavailable in practice (it needs the ground truth the metric is standing in for) and the
second targets a different objective, so neither is a like-for-like control.

\begin{table*}[t]
\centering
\caption{F1 under three inner-loop selection criteria, on the same 10 IHDP splits, the same
six estimators, and the same $3\times3$ fold protocol. Rank correlations are with $-\pehe$
across models, averaged over splits with a dataset-level bootstrap CI; $\Delta$ is the paired
AUUC-over-Qini gap. The released panel is Qini-tuned; the second arm removes tuning; the third
selects candidates by AUUC at the released budget. Qini shows no detectable alignment in every
arm and $\Delta$ stays positive with its CI excluding zero, so F1 does not depend on which
metric selects the candidates.}
\label{tab:tunearms}
\small
\begin{tabular}{lccc}
\toprule
Inner objective & $\rho$(Qini) & $\rho$(AUUC) & $\Delta$ (paired) \\
\midrule
Qini (released) & $\relQini$ [$\relQiniLo,\relQiniHi$] & $\relAuuc$
  [$\relAuucLo,\relAuucHi$] & $\relDelta$ [$\relDeltaLo,\relDeltaHi$] \\
none (defaults) & $\notuneQini$ [$\notuneQiniLo,\notuneQiniHi$] & $\notuneAuuc$
  [$\notuneAuucLo,\notuneAuucHi$] & $\notuneDelta$ [$\notuneDeltaLo,\notuneDeltaHi$] \\
AUUC & $\atQini$ [$\atQiniLo,\atQiniHi$] & $\atAuuc$ [$\atAuucLo,\atAuucHi$] &
  $\atDelta$ [$\atDeltaLo,\atDeltaHi$] \\
\bottomrule
\end{tabular}
\end{table*}

\paragraph{Pooled statistic (why we do not lead with it).} Collapsing
Qini-vs-reference-objective into a single mean rank-correlation gives $-0.18$ over the 21
datasets with a reference objective, but its 95\% bootstrap CI $[-0.37,+0.01]$ crosses zero,
and pooling continuous- and binary-outcome datasets conflates the two findings. The
decomposed, per-finding statistics in the main text are the paper's quantitative claims;
the pooled number is reported only for completeness.

\paragraph{F2 selector ladder, full detail.} The oracle bound (selecting on the evaluation
seed itself) gains $\selOracleGain$; the reference (risk-minimizing) winner is stable across
rotations with modal agreement $\selRefModalAgree\%$; Qini and AUUC pick identical winners
on Jobs, consistent with their mutual agreement on binary outcomes. Metric gains over
random: Qini $\selQiniGain$ [$\selQiniGainLo,\selQiniGainHi$], AUUC $\selAuucGain$
[$\selAuucGainLo,\selAuucGainHi$], uplift-at-$k$ $\selUpliftGain$
[$\selUpliftGainLo,\selUpliftGainHi$].

\paragraph{F2 regret, secondary statistics.} Per-metric regrets: Qini $\regQini$
[$\regQiniLo,\regQiniHi$], AUUC $\regAuuc$ [$\regAuucLo,\regAuucHi$], uplift-at-$k$
$\regUplift$ [$\regUpliftLo,\regUpliftHi$] (cluster bootstrap over $\regN$ Jobs splits;
random baseline $\regBase$). The $\regRelLo$--$\regRelHi\%$ relative figure is the regret
as a fraction of the \emph{mean} best-candidate risk across rotations (the per-rotation
best risk ranges $\regBestLo$--$\regBestHi$). Paired
differences vs.\ random span zero for all three metrics ($[\regDiffLo,\regDiffHi]$,
$[\regDiffAuucLo,\regDiffAuucHi]$, $[\regDiffUpliftLo,\regDiffUpliftHi]$); the
metric-selected model is within $0.005$/$0.01$/$0.02$ risk of the reference only
$\regWithinA\%$/$\regWithinB\%$/$\regWithinC\%$ of the time.

\paragraph{F2 rank correlations, read carefully.} The regret is mirrored by negative rank
correlations with policy \emph{value}: AUUC and uplift-at-$k$ correlate at $-0.24$
[$-0.44,-0.02$] and $-0.30$ [$-0.51,-0.05$] (excluding zero), while Qini is $-0.22$ but
borderline ($[-0.43,+0.02]$). F2 therefore does not rest on the Qini correlation alone; the
regret and selector-ladder results establish the gap for all three metrics directly. By the
appropriate metric the \model{DR-Learner} and \model{S-Learner} give the lowest Jobs policy
risk, in a narrow band of $0.21$--$0.25$ (leaderboards in Appendix~\ref{app:tables}).

\paragraph{causalml implementation comparison, full detail.} Over the $\ihdpN$
continuous-outcome datasets (all IHDP realizations),
\texttt{causalml}'s shipped \texttt{qini\_score} (v0.16.0, default normalization) and our
canonical unnormalized definition produce near-identical model rankings (mean per-dataset
Spearman $\cmlRankRho$ [$\cmlRankRhoLo,\cmlRankRhoHi$]; Qini-best model agreement
$\cmlWinnerAgree\%$), and the F1 statistic is unchanged: $\canMOne$
[$\canMOneLo,\canMOneHi$] canonical vs.\ $\cmlMOne$ [$\cmlMOneLo,\cmlMOneHi$] under
\texttt{causalml}. The silently computed object practitioners receive behaves exactly as
the object we analyze.

\paragraph{Budget grid, per-budget values.} Regret at each budget: $\kregUpliftZeroOne$
[$\kregUpliftZeroOneLo,\kregUpliftZeroOneHi$] ($k{=}0.1$), $\kregUpliftZeroTwo$
[$\kregUpliftZeroTwoLo,\kregUpliftZeroTwoHi$] ($k{=}0.2$), $\kregUpliftZeroThree$
[$\kregUpliftZeroThreeLo,\kregUpliftZeroThreeHi$] ($k{=}0.3$), $\kregUpliftZeroFive$
[$\kregUpliftZeroFiveLo,\kregUpliftZeroFiveHi$] ($k{=}0.5$); PV-AUC selector $\kregPVAUC$
[$\kregPVAUCLo,\kregPVAUCHi$]; rank stability across budgets $\kgridStab$
[$\kgridStabLo,\kgridStabHi$]. Value-reference rotation at $k{=}0.3$: Qini regret
$\bOneQiniVrefReg$ [$\bOneQiniVrefRegLo,\bOneQiniVrefRegHi$], gain over random
$\bOneQiniVrefGain$ [$\bOneQiniVrefGainLo,\bOneQiniVrefGainHi$].

\paragraph{Estimation noise and split dependence (full detail).}
Two structural caveats bound what the F2 intervals can claim, and we measure both. First,
the IPW policy risk is a noisy estimand on the small experimental subset: unit-level
bootstrapping the experimental units within each evaluation fold (predictions held fixed)
gives a per-split risk standard error with median $\riskUnitSeMed$
(10--90\% range $[\riskUnitSeLo,\riskUnitSeHi]$) --- the same order as the regret itself, so
no single split is informative; the inference runs entirely through pairing and
averaging (3 rotations $\times$ 10 splits). Second, because the Jobs splits
re-partition the same observations, the cluster bootstrap's exchangeability assumption may
understate uncertainty. A design-effect sensitivity makes the assumption's role explicit:
writing $\mathrm{SE}_{\mathrm{true}}=\mathrm{SE}_{\mathrm{iid}}\sqrt{1+(n{-}1)\bar\rho}$
for average between-split dependence $\bar\rho$, the 95\% regret interval continues to
exclude zero only for $\bar\rho\le\regRhoStarQini$ (Qini and AUUC;
$\regRhoStarUplift$ for uplift-at-$k$); the leave-one-split-out jackknife matches the iid
SE, and the regret is positive in $\regPosQini$/10 splits. The same caveat applies to the
selector-ladder gains above. We therefore read F2 as a \emph{consistently positive regret of
modest statistical strength}, resting on the convergence of four diagnostics ---
convergent, not statistically independent, since all four read the same underlying
observations ---
the rotation regret, the sign pattern, the negative rank correlations (which replicate under
XGBoost), and the risk-vs-metric selector contrast --- rather than on any single interval.

\paragraph{Replacement metrics (full detail).} We add the rank-weighted average
treatment effect \citep{yadlowsky2025rate} to the harness (IPW effect scores from the
fold propensities; AUTOC and Qini weightings) and score the same stored predictions. On the
IHDP continuous panel RATE improves on Qini but only marginally: rank correlation
with $-\pehe$ is $\rateAutocMOne$ [$\rateAutocMOneLo,\rateAutocMOneHi$] (AUTOC) and
$\rateQiniMOne$ [$\rateQiniMOneLo,\rateQiniMOneHi$] (Qini-weighted) --- statistically
positive, unlike Qini's $\canMOne$, yet far below AUUC's $\auucMOne$
[$\auucMOneLo,\auucMOneHi$] on the identical frame. RATE alone does not restore
effect-accuracy tracking here. Nor does variance reduction: recomputing Qini on
\emph{oracle} baseline-adjusted outcomes $y-\mu_0(x)$ (an optimistic benchmark for
outcome-residualization-based variance reduction, cf.\ the corollary to
Lemma~\ref{lem:baseline}) lifts the correlation only to
$\bFourAdj$ [$\bFourAdjLo,\bFourAdjHi$] on the $\bFourN$ IHDP realizations --- the same
modest level as RATE, and far from AUUC --- so estimation variance of the Qini functional is
not the driver of F1. These two probes do not exhaust the proposed replacements: PUC,
pROCini, and the exact variance-reduced estimator of \citet{bokelmann2024improving} remain
unbenchmarked (the oracle adjustment upper-bounds the outcome-residualization class
considered here, which is the question relevant to F1).

\paragraph{Mechanism-candidate sweeps (full detail).}
A concrete mechanism candidate for F1 lives in Lemma~\ref{lem:baseline}: the interleaving
term proportional to $T_k-\tfrac{k}{n}T_n$ multiplies the full-sample ATE estimate, whose
noise is heavy-tailed in the controlled sweep, and treatment \emph{imbalance} (IHDP:
$\pi_1{=}0.19$) inflates those count fluctuations. We tested it directly: a controlled sweep
of $\pi_1\in\{0.50,0.35,0.19,0.10\}$ crossed with the four outcome regimes (identical latent
CATE and model panel; $\piSweepNSeeds$ seeds per cell). The result is null in all 16 cells:
the paired Qini$-$AUUC fidelity gap is never significantly negative --- at the IHDP-matched
cell ($\pi_1{=}0.19$, heavy log-normal) it is $\piSweepImbDelta$
[$\piSweepImbDeltaLo,\piSweepImbDeltaHi$] --- so imbalance $\times$ tails does not reproduce
the benchmark's Qini isolation either. This tightens the boundary of
Section~\ref{sec:m1boundary}: scale artifacts, single estimators, base learners,
implementation variants, kurtosis alone, and now imbalance $\times$ kurtosis are all
eliminated as sole drivers. We then tested the last candidate we considered plausible ---
score errors \emph{correlated with the outcome model} (heteroskedastic in the realized
potential-outcome magnitude), mixed into the panel at $\lambda\in[0,1]$ and crossed with
the kurtosis ladder ($\correrrNSeeds$ seeds/cell). It is also null, and informatively so:
under correlated errors and heavy tails the paired gap moves in Qini's \emph{favor}
($\correrrKeyDelta$ [$\correrrKeyDeltaLo,\correrrKeyDeltaHi$] at $\lambda{=}0.75$,
heavy log-normal), never reproducing the benchmark's isolation in any cell. With scale,
kurtosis alone, imbalance $\times$ kurtosis, correlated errors $\times$ kurtosis, and
estimation variance (the oracle-adjusted benchmark above) all eliminated, we position F1 as a
benchmark-discovered failure signature whose isolating mechanism remains open despite
targeted tests --- the signature itself is what the multi-regime benchmark contributes.

\paragraph{Tuning budget (full detail).} A bounded tuning budget ($B{=}10$) compresses the quality
spread of the candidate panel, which can attenuate \emph{all} model-level correlations and
makes small-panel Spearman statistics noisier in both directions; we do not claim limited
tuning cannot interact with individual metrics. Two design features bound the concern: every
metric scores the \emph{identical} out-of-fold predictions, and the load-bearing statistic
is the paired AUUC$-$Qini gap $\Delta$ with its CI --- robust across base learners and
estimator exclusions --- not any absolute correlation level.

\section{Calibration: pooled contrast and bin-count robustness}
\label{app:calibration}

This appendix supports Section~\ref{sec:calibration}. Restricting to the $\calNident$
randomized instances (where the treated-minus-control bin contrast identifies uplift), Qini
and $-$ECE correlate at $\calRhoId$ [$\calRhoIdLo, \calRhoIdHi$] and the best-Qini and best-ECE
models differ in $\calDisIdN$/$\calDisIdD$ instances with a small paired ECE cost (median
$\calDeltaMed$). Pooling over all $25$ instances --- \emph{including} the confounded IHDP/Jobs
datasets, where the within-bin contrast conflates uplift with selection and the ECE is biased
--- instead drives the correlation to $\calRhoAll$ [$\calRhoAllLo, \calRhoAllHi$] and inflates
the winner disagreement to $\calDisAllN$/$\calDisAllD$. The apparent pooled ``disagreement'' is
thus driven by applying an unidentified within-bin contrast to confounded datasets, not by a
Qini--calibration gap.
The ECE-induced ranking is itself bin-count-robust: on a controlled panel ($\calPanelModels$
models, $\calPanelSeeds$ seeds) the rankings at $5$, $10$ and $20$ bins are identical, with
every pairwise Spearman correlation equal to $\calBinHi$.

\section{Leaderboards and tables}
\label{app:tables}

\begin{figure*}[t]
  \centering
  \includegraphics[width=\linewidth]{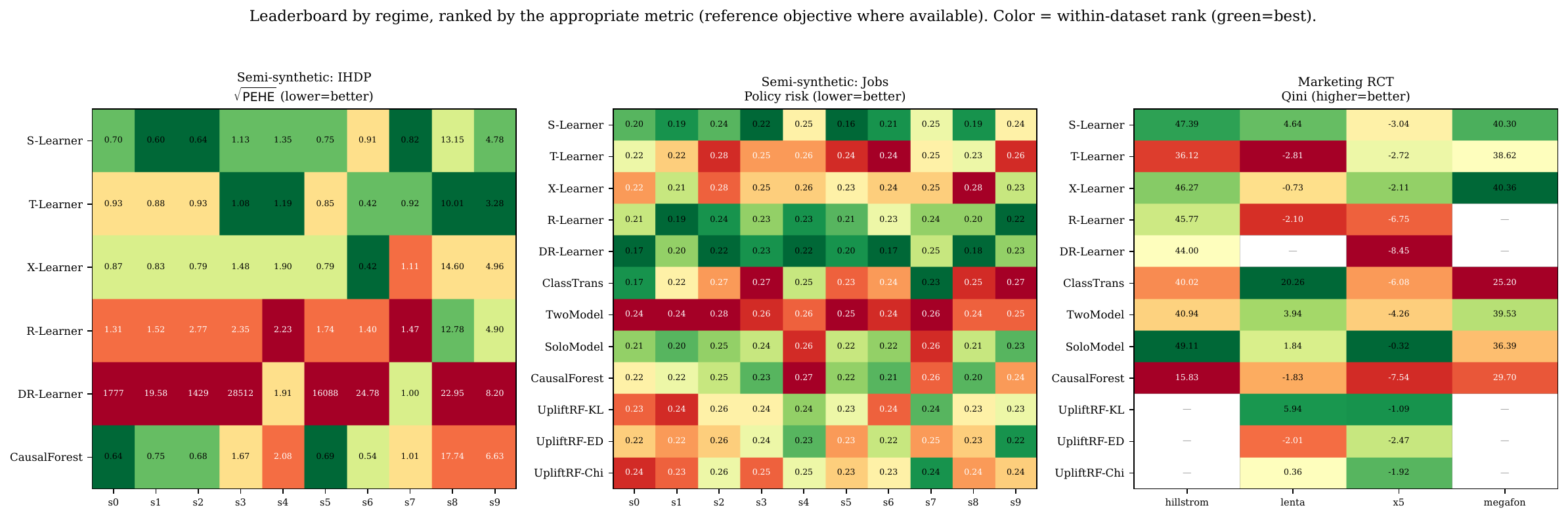}
  \Description{Per-regime leaderboard heatmap; cell color encodes within-dataset model rank by the regime's appropriate metric.}
  \caption{Per-regime leaderboard, each panel scored by its appropriate metric (color =
  within-dataset rank, green best); ``---'' marks binary-only models on continuous IHDP.
  IHDP cells are the arithmetic mean of fold-level $\pehe$ (not $\sqrt{\text{mean PEHE}}$),
  so single divergent folds inflate the affected DR-Learner cells; rankings are unaffected.}
  \label{fig:leaderboard}
\end{figure*}

\begin{figure*}[t]
  \centering
  \begin{subfigure}[b]{0.49\linewidth}
    \includegraphics[width=\linewidth]{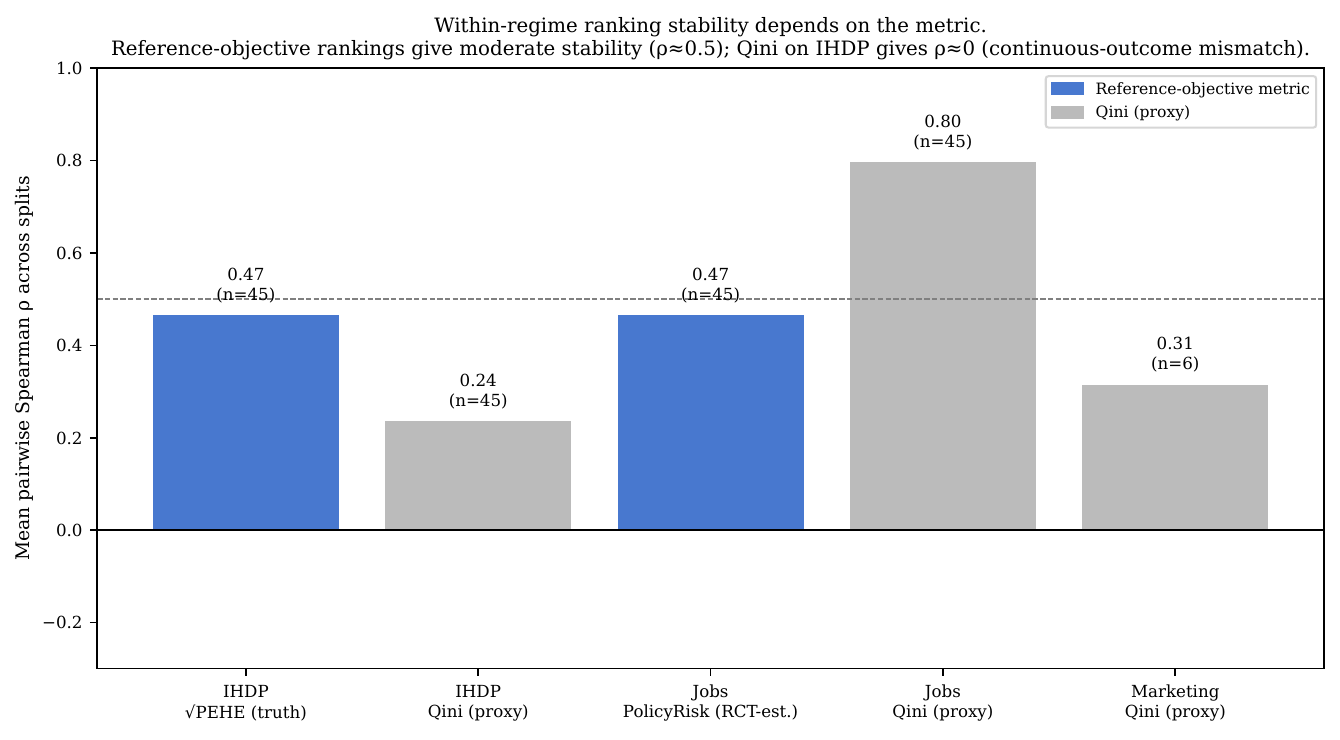}
  \Description{Bar chart of mean within-dataset rank stability by metric.}
    \caption{Within-regime rank stability by metric.}
  \end{subfigure}\hfill
  \begin{subfigure}[b]{0.49\linewidth}
    \includegraphics[width=\linewidth]{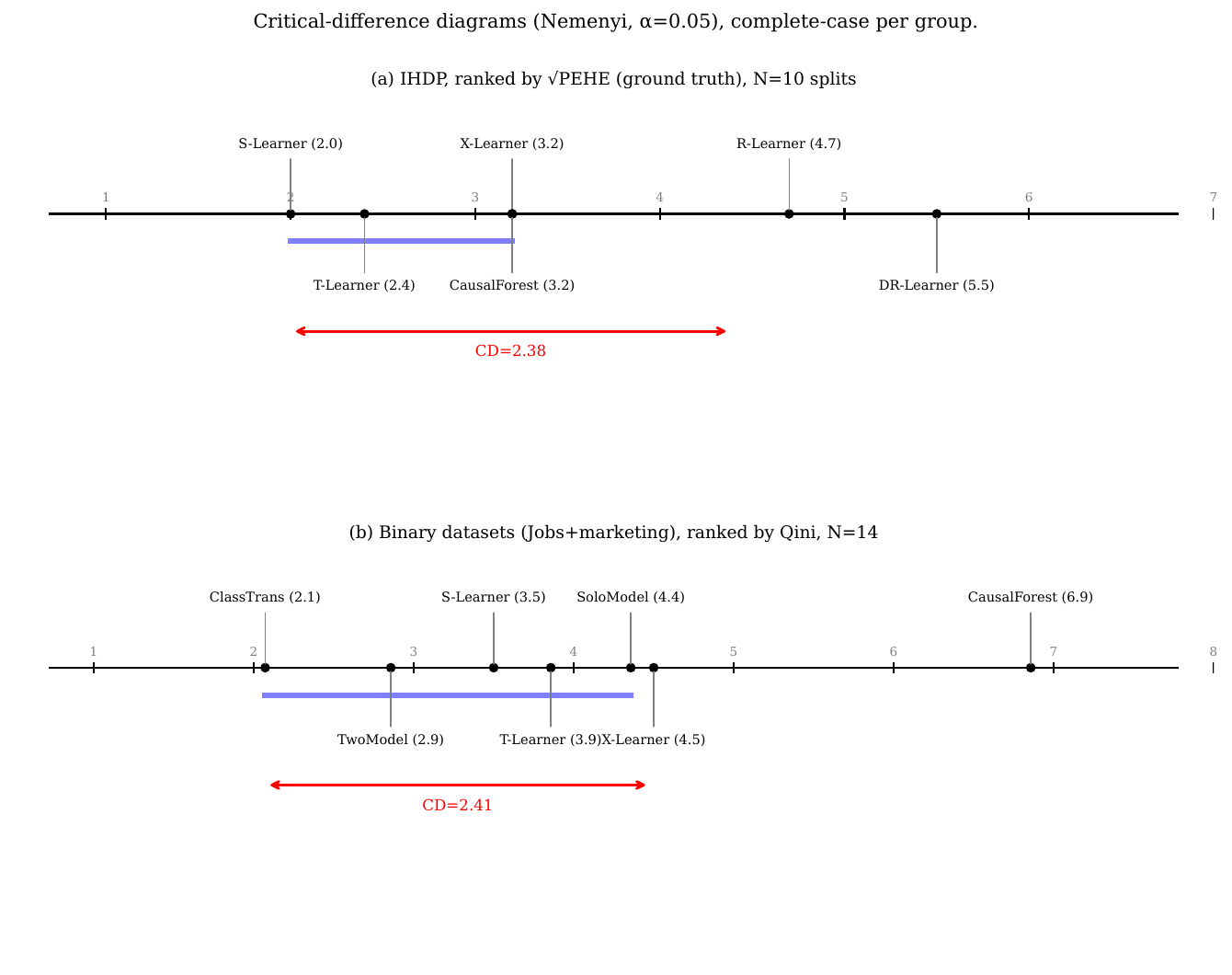}
  \Description{Nemenyi critical-difference diagrams over the estimators.}
    \caption{Complete-case critical-difference diagrams.}
  \end{subfigure}
  \caption{Supporting analyses for Section~\ref{sec:robustness}--\ref{sec:m2}.}
  \label{fig:supp}
\end{figure*}

Table~\ref{tab:disagree} lists, per dataset with a reference objective, the best model by that objective vs.\ by
Qini; Tables~\ref{tab:ece} and~\ref{tab:ecebin} report calibration ECE on the
continuous- and binary-outcome datasets respectively. ECE is tabulated per outcome family
because the IHDP and Jobs realization labels both abbreviate to s0--s9; the DR-Learner's
large IHDP ECE is consistent with the numerical instability documented above. Reporting
convention: leaderboard and ECE cells are means of the corresponding fold-level quantity
(for IHDP, of fold-level $\pehe$), so a single divergent fold can dominate a cell while
leaving every rank-based statistic unchanged.

\begin{table*}[t]
\centering
\caption{F1 estimator-exclusion sensitivity on the continuous-outcome datasets (IHDP $\times$10). $\rho$: mean within-dataset Spearman correlation with $-\pehe$. $\Delta$: paired per-dataset difference $\rho(\mathrm{AUUC})-\rho(\mathrm{Qini})$ with 95\% bootstrap CI. The AUUC advantage is positive in all six configurations; its CI excludes zero in five of six (the exception is XGBoost with DR- and R-Learner excluded).}
\label{tab:m1sens}
\small
\begin{tabular}{llrrrl}
\toprule
Base learner & Estimators & Qini $\rho$ & AUUC $\rho$ & $\Delta$ & 95\% CI \\
\midrule
LightGBM & six meta/forest estimators & +0.02 & +0.65 & +0.63 & [+0.35, +0.91] \\
LightGBM & DR-Learner excluded & +0.12 & +0.52 & +0.40 & [+0.10, +0.70] \\
LightGBM & DR- and R-Learner excluded & -0.08 & +0.34 & +0.42 & [+0.06, +0.80] \\
XGBoost & six meta/forest estimators & -0.21 & +0.48 & +0.69 & [+0.44, +0.90] \\
XGBoost & DR-Learner excluded & -0.03 & +0.41 & +0.44 & [+0.15, +0.73] \\
XGBoost & DR- and R-Learner excluded & -0.12 & +0.22 & +0.34 & [-0.16, +0.84] \\
\bottomrule
\end{tabular}
\end{table*}

\begin{table*}[t]
\centering
\caption{Affine invariance. Multiplying or shifting a Gaussian outcome leaves every metric's within-split model ranking (and thus its $\rho$ with truth) unchanged: the count/mean corrections cancel affine transforms. The F1 failure is therefore not driven by outcome scale or location; distributional shape remains relevant but does not by itself explain the observed Qini/AUUC separation.}
\label{tab:affine}
\small
\begin{tabular}{lrrr}
\toprule
Transform & Qini (raw) & AUUC & Uplift@$k$ \\
\midrule
$y$ & +0.91 & +0.89 & +0.85 \\
$10y$ & +0.91 & +0.89 & +0.85 \\
$y+100$ & +0.91 & +0.89 & +0.85 \\
$3y+50$ & +0.91 & +0.89 & +0.85 \\
\bottomrule
\end{tabular}
\end{table*}

\begin{table*}[t]
\centering
\caption{Controlled outcome-distribution experiment. Spearman $\rho$ (mean over 30 seeds, $\pm$ s.e.) between each metric's ranking of a fixed 8-model panel and the \emph{known} quality order ($-\pehe$ against the latent CATE), with the models and latent CATE held fixed as only the outcome distribution's tail varies. All cumulative ranking metrics lose fidelity as the tail thickens, and Qini and AUUC degrade \emph{together}: the AUUC-over-Qini advantage seen on IHDP (Sec.~\ref{sec:robustness}) is thus an empirical property of that dataset, not a property guaranteed by construction --- and on \dataset{Revenue-Synthetic}, whose outcome is far heavier-tailed, the advantage reverses.}
\label{tab:transform}
\small
\begin{tabular}{lrrrr}
\toprule
Outcome regime & excess kurt. & Qini (raw) & AUUC & Uplift@$k$ \\
\midrule
Gaussian & 0.2 & +0.91 & +0.89 & +0.85 \\
Student-$t_3$ & 7.5 & +0.85 & +0.83 & +0.76 \\
Log-normal (mod.) & 17.5 & +0.79 & +0.68 & +0.69 \\
Log-normal (heavy) & 70.3 & +0.48 & +0.36 & +0.43 \\
\bottomrule
\end{tabular}
\end{table*}

\begin{table*}[t]
\centering
\caption{Full data for the F1 existence counterexample (Section~\ref{sec:m1}). Twelve units; $t$ treatment, $y$ observed outcome (one large control outcome at unit~3), $\tau$ true CATE; models A and B are the scored uplift predictions, with their induced descending-sort ranks. Metrics computed by the shipped functions give $Q_A{=}\cxQiniA{>}Q_B{=}\cxQiniB$ (Qini prefers the worse A) but $\mathrm{AUUC}_A{=}\cxAuucA{<}\mathrm{AUUC}_B{=}\cxAuucB$ and $\pehe_A{=}\cxPeheA{>}\pehe_B{=}\cxPeheB$ (both prefer B). Setting $y$ at unit~3 to $0$ yields $Q_A{=}\cxQiniAdrop{<}Q_B{=}\cxQiniBdrop$.}
\label{tab:cxvectors}
\small
\begin{tabular}{r rrr rr rr}
\toprule
unit & $t$ & $y$ & $\tau$ & score A & rank A & score B & rank B \\
\midrule
1 & 0 & 0 & 1 & 0.087 & 11 & 1.077 & 8 \\
2 & 0 & 0 & 1 & 0.127 & 9 & 0.999 & 9 \\
3 & 0 & 40 & 2 & 0.753 & 4 & 2.026 & 5 \\
4 & 1 & 1 & 1 & 0.241 & 8 & 0.995 & 10 \\
5 & 1 & 0 & 0 & 0.800 & 2 & -0.084 & 12 \\
6 & 1 & 2 & 2 & 0.094 & 10 & 2.056 & 3 \\
7 & 0 & 0 & 2 & 0.771 & 3 & 2.037 & 4 \\
8 & 1 & 2 & 2 & 0.443 & 6 & 2.060 & 2 \\
9 & 0 & 0 & 0 & 0.804 & 1 & -0.053 & 11 \\
10 & 1 & 2 & 2 & 0.015 & 12 & 2.007 & 6 \\
11 & 0 & 0 & 2 & 0.751 & 5 & 2.095 & 1 \\
12 & 0 & 0 & 1 & 0.411 & 7 & 1.110 & 7 \\
\bottomrule
\end{tabular}
\end{table*}

\begin{table*}[t]
\centering
\caption{Documentation of the all-100 IHDP validation (Section~\ref{sec:robustness}). Splits 0--9 come from the main benchmark, 10--29 from the F1-extension run, and 30--99 from a dedicated validation run; all use the same loaders, LightGBM base learner, outer-test-isolated protocol and $3\times3$ outer CV. The validation run uses a smaller tuning budget ($B{=}6$ rather than $B{=}10$) and restricts to the six meta/forest learners common to every run. Correlations are complete-case per split (a split enters if $\ge 4$ models have a defined score).}
\label{tab:ihdp100doc}
\small
\begin{tabular}{lll}
\toprule
Item & Value & Notes \\
\midrule
Realizations & 100 (splits 0--99) & all completed \\
Estimators & S/T/X/R/DR-Learner, Causal Forest & common to all three runs \\
Base learner & LightGBM & identical across splits \\
Outer CV & $3$ seeds $\times\,3$ folds & identical across splits \\
Tuning budget & $B{=}10$ (splits 0--29), $B{=}6$ (30--99) & differs across runs \\
Models per split correlation & 6--6 (median 6) & complete-case \\
Missing-data rule & split included if $\ge4$ models scored & no imputation \\
\bottomrule
\end{tabular}
\end{table*}

\begin{table*}[t]
\centering
\caption{Audit of the F2 Jobs policy-regret experiment under both base learners (Section~\ref{sec:m2}). Regret is the cross-repeat selection/evaluation rotation regret (mean over the ten Jobs splits, 95\% cluster bootstrap); $\rho(\text{Qini},-R_{\mathrm{policy}})$ is the mean per-split Spearman correlation with policy value. Both halves of F2 --- the ranking metrics' negative association with policy value and the positive selection regret --- replicate under XGBoost (indeed Qini's correlation is no longer borderline there).}
\label{tab:regretaudit}
\small
\begin{tabular}{lll}
\toprule
Item & LightGBM & XGBoost \\
\midrule
Jobs splits & 10 & 10 \\
Candidate models & 12 & 12 \\
Completed models / split & 12--12 & 12--12 \\
Qini regret & +0.028 [+0.009, +0.046] & +0.022 [+0.011, +0.032] \\
AUUC regret & +0.028 [+0.009, +0.046] & +0.017 [+0.005, +0.028] \\
Uplift-at-$k$ regret & +0.026 [+0.006, +0.045] & +0.023 [+0.012, +0.034] \\
$\rho(\text{Qini},-R_{\mathrm{policy}})$ & -0.24 [-0.47, +0.01] & -0.35 [-0.53, -0.18] \\
Bootstrap unit & split & split \\
\bottomrule
\end{tabular}
\end{table*}

\begin{table*}[t]
\centering
\small
\caption{Best model by reference objective vs by Qini, and their rank correlation. The reference objective is ground-truth $\sqrt{\mathrm{PEHE}}$ on IHDP and \emph{RCT-estimated} policy risk on Jobs (no per-unit ground truth exists on Jobs).} \label{tab:disagree}
\begin{tabular}{lrrr}
\toprule
Dataset & Best by reference objective & Best by Qini & rank-corr \\
\midrule
IHDP-s0 & CausalForest & T-Learner & -0.26 \\
IHDP-s1 & S-Learner & DR-Learner & -0.31 \\
IHDP-s2 & S-Learner & X-Learner & -0.14 \\
IHDP-s3 & T-Learner & S-Learner & 0.71 \\
IHDP-s4 & T-Learner & DR-Learner & 0.49 \\
IHDP-s5 & CausalForest & DR-Learner & -0.83 \\
IHDP-s6 & X-Learner & X-Learner & 0.37 \\
IHDP-s7 & S-Learner & S-Learner & 0.89 \\
IHDP-s8 & T-Learner & DR-Learner & -0.83 \\
IHDP-s9 & T-Learner & CausalForest & 0.09 \\
Jobs-s0 & DR-Learner & ClassTrans & -0.18 \\
Jobs-s1 & R-Learner & ClassTrans & 0.18 \\
Jobs-s2 & DR-Learner & ClassTrans & -0.55 \\
Jobs-s3 & S-Learner & ClassTrans & -0.70 \\
Jobs-s4 & DR-Learner & ClassTrans & 0.31 \\
Jobs-s5 & S-Learner & ClassTrans & -0.57 \\
Jobs-s6 & DR-Learner & ClassTrans & -0.50 \\
Jobs-s7 & ClassTrans & ClassTrans & 0.45 \\
Jobs-s8 & DR-Learner & ClassTrans & -0.52 \\
Jobs-s9 & R-Learner & ClassTrans & -0.36 \\
synthetic & CausalForest & TwoModel & -0.22 \\
\bottomrule
\end{tabular}
\end{table*}

\begingroup
\tiny
\setlength{\tabcolsep}{1pt}
\begin{table*}[t]
\centering
\small
\caption{Calibration ECE, continuous-outcome datasets (IHDP s0--s9); mean across folds, lower=better. Uplift ECE is $\sum_b w_b\lvert \bar s_b - \widehat{\tau}_b\rvert$ on the \emph{score} scale, so it is unbounded above: the large DR-Learner entries inherit the same AIPW pseudo-outcome divergence documented for $\pehe$ (Appendix~\ref{app:m1detail}), not a separate defect.} \label{tab:ece}
\resizebox{\textwidth}{!}{%
\begin{tabular}{lrrrrrrrrrr}
\toprule
Model & s0 & s1 & s2 & s3 & s4 & s5 & s6 & s7 & s8 & s9 \\
\midrule
S-Learner & 0.629 & 0.583 & 0.591 & 0.649 & 0.590 & 0.675 & 0.811 & 0.601 & 5.261 & 2.281 \\
T-Learner & 0.683 & 0.697 & 0.655 & 0.659 & 0.613 & 0.652 & 0.514 & 0.727 & 2.309 & 1.223 \\
X-Learner & 0.701 & 0.634 & 0.577 & 0.681 & 1.010 & 0.663 & 0.489 & 0.854 & 6.411 & 1.958 \\
R-Learner & 0.968 & 1.124 & 1.753 & 1.388 & 1.033 & 1.154 & 1.175 & 0.947 & 4.395 & 1.769 \\
DR-Learner & 1283.585 & 13.437 & 1004.196 & 18743.163 & 0.718 & 11055.110 & 14.521 & 0.619 & 11.985 & 4.357 \\
CausalForest & 0.630 & 0.673 & 0.595 & 0.892 & 1.046 & 0.637 & 0.552 & 0.656 & 7.988 & 3.054 \\
\bottomrule
\end{tabular}}
\end{table*}

\begin{table*}[t]
\centering
\small
\caption{Calibration ECE, binary-outcome datasets (Jobs s0--s9, Synthetic, marketing RCTs); mean across folds, lower=better. Values are not bounded by 1 even for binary outcomes, because the predicted uplift entering the score-scale ECE is unbounded; the two DR-Learner outliers (Jobs s7, s9) are that divergence. Every ECE statistic quoted in the text is a \emph{median} paired cost, so these cells do not drive any reported number.} \label{tab:ecebin}
\resizebox{\textwidth}{!}{%
\begin{tabular}{lrrrrrrrrrrrrrrr}
\toprule
Model & s0 & s1 & s2 & s3 & s4 & s5 & s6 & s7 & s8 & s9 & synthetic & hillstrom & lenta & x5 & megafon \\
\midrule
S-Learner & 0.172 & 0.202 & 0.157 & 0.146 & 0.182 & 0.149 & 0.135 & 0.180 & 0.151 & 0.148 & 0.097 & 0.019 & 0.013 & 0.089 & 0.048 \\
T-Learner & 0.197 & 0.229 & 0.177 & 0.157 & 0.182 & 0.207 & 0.182 & 0.168 & 0.155 & 0.157 & 0.212 & 0.032 & 0.094 & 0.088 & 0.080 \\
X-Learner & 0.176 & 0.206 & 0.148 & 0.150 & 0.157 & 0.201 & 0.126 & 0.154 & 0.165 & 0.196 & 0.192 & 0.019 & 0.064 & 0.085 & 0.050 \\
R-Learner & 0.267 & 0.304 & 0.268 & 0.295 & 0.275 & 0.272 & 0.273 & 0.314 & 0.274 & 0.278 & 0.298 & 0.020 & 0.134 & 0.102 & --- \\
DR-Learner & 0.214 & 0.224 & 0.190 & 0.808 & 0.244 & 0.187 & 0.193 & 996.263 & 0.201 & 365.968 & 409.684 & 0.020 & --- & 0.125 & --- \\
ClassTrans & 0.499 & 0.511 & 0.490 & 0.543 & 0.541 & 0.544 & 0.524 & 0.533 & 0.542 & 0.535 & 0.387 & 0.243 & 0.407 & 0.091 & 0.154 \\
TwoModel & 0.213 & 0.235 & 0.177 & 0.184 & 0.179 & 0.196 & 0.176 & 0.171 & 0.168 & 0.185 & 0.317 & 0.023 & 0.063 & 0.100 & 0.059 \\
SoloModel & 0.159 & 0.198 & 0.158 & 0.116 & 0.190 & 0.139 & 0.118 & 0.153 & 0.140 & 0.149 & 0.094 & 0.019 & 0.017 & 0.070 & 0.051 \\
CausalForest & 0.267 & 0.582 & 0.672 & 0.396 & 0.517 & 0.584 & 0.523 & 0.723 & 0.363 & 0.551 & 0.104 & 0.044 & 0.038 & 0.134 & 0.045 \\
UpliftRF-KL & 0.223 & 0.217 & 0.184 & 0.150 & 0.198 & 0.203 & 0.176 & 0.190 & 0.161 & 0.186 & --- & --- & 0.030 & 0.064 & --- \\
UpliftRF-ED & 0.220 & 0.236 & 0.204 & 0.172 & 0.185 & 0.210 & 0.181 & 0.196 & 0.201 & 0.194 & --- & --- & 0.037 & 0.065 & --- \\
UpliftRF-Chi & 0.228 & 0.229 & 0.197 & 0.163 & 0.204 & 0.215 & 0.159 & 0.175 & 0.182 & 0.203 & --- & --- & 0.035 & 0.064 & --- \\
\bottomrule
\end{tabular}}
\end{table*}

\endgroup

\section{Run statistics}
Table~\ref{tab:runstats} accounts for every fold-level evaluation and every
dataset$\times$model cell in the released run (completed, correctly skipped, timed out).
The \dataset{Criteo} supplement (Appendix~\ref{app:criteo}) sits outside this ledger by
design: it is released and versioned, but not part of the 300-cell accounting.
\begin{table*}[t]
\centering
\caption{Benchmark run statistics (current data).}
\label{tab:runstats}
\begin{tabular}{lr}
\toprule
Statistic & Value \\
\midrule
Fold-level evaluations produced & 2{,}112 \\
\quad successful & 2{,}052 \\
\quad correctly skipped (binary-only model, continuous IHDP) & 60 \\
Dataset $\times$ model cells & 300 \\
\quad completed / skipped / timed out & 228 / 60 / 12 \\
Base learners & LightGBM, XGBoost \\
\bottomrule
\end{tabular}
\end{table*}

\clearpage
\section{The Criteo supplement}
\label{app:criteo}
\dataset{Criteo} Uplift v2.1 \citep{diemert2018large} is the field's largest public uplift
RCT: 13.9M rows, 12 features, treatment fraction $\pi_1{=}0.85$, binary \texttt{visit}
outcome with base rate $0.047$ (1M tier). The supplement (repository v1.4.3,
\texttt{results\_criteo/} and \texttt{results\_criteo100k/}) runs all 12 estimators under
the released protocol on the 1M tier subsampled to the 10K evaluation cap for comparability
with the other large RCTs, plus a 100K probe of nine of the twelve estimators (the three
uplift forests are omitted on compute grounds: ${\sim}17$ minutes per cell already at 10K,
scaling superlinearly) --- the
largest-scale released result. It sits outside the 25-instance\slash 300-cell ledger and
enters neither F1 nor F2 (no ground-truth effects). Table~\ref{tab:criteo} reports
per-model Qini with fold-level standard errors at both scales.

\paragraph{Subsampling design.} The cap subsamples the \emph{entire dataset} before outer
folding (each tier's $n$ is then split by the standard $3{\times}3$ repeated CV, so a 10K
tier has ${\approx}6{,}666$ training and ${\approx}3{,}334$ test rows per fold), stratified
jointly by treatment $\times$ outcome via a seeded \texttt{StratifiedShuffleSplit}. Both
tiers draw from the same 1M tier with the same seed and identical fold protocol, but the
10K draw is \emph{not nested} in the 100K draw (stratified draws are not prefix-stable
across sizes) --- so the cross-tier rank comparison spans one pair of separately
generated, non-nested stratified subsamples, and training and evaluation sizes change
together across tiers.
Within a tier, all estimators share identical subsampled rows and identical folds, which
is what licenses the paired per-fold differences below.

\paragraph{Replication (descriptive).} The binary-regime cross-metric rank agreement
appears on Criteo as well (Qini--AUUC $+0.74$, Qini--uplift@$k$ $+0.80$,
AUUC--uplift@$k$ $+0.92$ across the 12
estimators --- inside the released binary-panel range); this is a descriptive external
replication on one fixed stratified subsample, not a formal test. \model{ClassTrans} posts
the highest point estimate at the cap ($+11.1 \pm 0.9$), a lead over \model{CausalForest}
that is within fold-level noise (paired per-fold difference $+6.0 \pm 4.5$, positive on 6
of 9 folds); its fold-level
variance is substantially smaller than most alternatives' --- a single classifier on the
transformed label, with no nuisance stacking. At 100K the same two lead in the same order
(\model{ClassTrans} $+96.9 \pm 6.0$ vs.\ \model{CausalForest} $+88.1 \pm 8.8$; paired
$+8.8 \pm 11.0$, positive on 6 of 9 folds): the capped point-estimate winner survives the
tenfold resolution increase, though the two leaders remain within noise of each other at
both tiers.

\paragraph{Resolution.} At the 10K cap,
pairwise gaps among the six meta\slash forest learners were small relative to
fold-to-fold variation --- descriptive non-resolution, not an inferential equivalence
result --- and their rankings did not agree between the
10K and 100K probes (Spearman $-0.03$ across those six;
\model{X-Learner} moves last${\to}$second, \model{R-Learner} third${\to}$last); over all
nine common estimators the correlation is $+0.20$, the agreement carried by the two
leaders. Because this is one subsample comparison and training size also changed, we read
it as evidence consistent with limited resolution at the cap --- not as the cap actively
misranking. At 100K the signal-to-noise ratio improves sharply (the
top point estimates' fold-level $t$ rises from ${\approx}1.2$ to ${\approx}10$), and
paired per-fold differences on identical folds (descriptive --- repeated-CV folds are
dependent, so these are not formal tests) put \model{CausalForest} ahead of every
meta-learner on 6--7 of 9 folds (mean differences $+19.6$ to $+44.1$), while the
\model{ClassTrans}--\model{CausalForest} pair stays unresolved at both tiers.

\paragraph{Confound and mechanism.} Training and
evaluation samples grew together, so isolating evaluation-side resolution would need a
train-at-100K\slash evaluate-at-10K contrast, which we have not run. Criteo is especially
cap-limited by design: $\pi_1{=}0.85$ with a control visit rate of $0.038$ leaves
${\approx}19$ positive control events per test fold at the cap --- supporting evidence for
the resolution reading.

\paragraph{Scope and licensing.} This is why every capped leaderboard in this paper is
stated as
subsample-scoped. \texttt{scripts/leaderboard\_resolution.py} reproduces the analysis for
all marketing RCTs: \dataset{Lenta}'s leader shows a large descriptive separation at the
cap (${>}5$ combined SEs) while the \dataset{Hillstrom}\slash\dataset{X5}\slash\dataset{MegaFon} leads
are within fold-to-fold noise. Criteo Uplift v2.1 is distributed under Criteo Research's own
research-use terms; the harness downloads it from Criteo's server at build time and never
redistributes it.

\begin{table}[ht]
\centering
\caption{Criteo supplement: per-model Qini $\pm$ fold-level, descriptive SE at the 10K
cap (12 estimators) and the 100K probe (nine; uplift forests omitted for compute).
Unnormalized Qini scales with $n$: compare orderings within a tier only. Paired per-fold
differences are in the text.}
\label{tab:criteo}
\small
% auto-generated by scripts/leaderboard_resolution.py
\begin{tabular}{lrr}
\toprule
Model & Qini (10K cap) $\pm$ fold SE & Qini (100K probe) $\pm$ fold SE \\
\midrule
ClassTrans & $+11.09 \pm 0.89$ & $+96.9 \pm 6.0$ \\
CausalForest & $+5.14 \pm 4.24$ & $+88.1 \pm 8.8$ \\
TwoModel & $+4.23 \pm 4.45$ & $+61.7 \pm 4.5$ \\
T-Learner & $+3.34 \pm 2.10$ & $+50.8 \pm 11.7$ \\
R-Learner & $+2.81 \pm 5.76$ & $+44.0 \pm 13.8$ \\
SoloModel & $+1.68 \pm 3.90$ & $+67.7 \pm 6.1$ \\
UpliftRF-Chi & $-0.10 \pm 3.01$ & --- \\
S-Learner & $-0.52 \pm 2.21$ & $+65.4 \pm 15.6$ \\
DR-Learner & $-0.70 \pm 3.41$ & $+63.1 \pm 7.5$ \\
UpliftRF-ED & $-1.37 \pm 2.21$ & --- \\
UpliftRF-KL & $-1.75 \pm 3.13$ & --- \\
X-Learner & $-2.31 \pm 4.67$ & $+68.5 \pm 10.3$ \\
\bottomrule
\end{tabular}

\end{table}

\end{document}